\documentclass[11pt]{amsart}
\usepackage{amssymb}
\usepackage{epsfig}  		
\usepackage{epic,eepic}       
\usepackage{amsmath}
\usepackage{amsmath,amsfonts,amssymb}
\usepackage{graphicx}
\usepackage{amscd}
\usepackage{pb-diagram}
\usepackage{url}
\newtheorem{theorem}{Theorem}[section]

\newtheorem{lemma}[theorem]{Lemma}

\usepackage[usenames,dvipsnames]{xcolor}
\usepackage{color}

\theoremstyle{definition}
\newtheorem{definition}[theorem]{Definition}
\newtheorem{example}[theorem]{Example}
\theoremstyle{remark}

\numberwithin{equation}{section}
\numberwithin{theorem}{section}
\numberwithin{table}{section}

\usepackage[noend]{algorithmic}
\usepackage{algorithm}
\floatname{algorithm}{Algorithm}

\newcommand{\cardd}[1]{\ensuremath{\left\|#1\right\|}}

\begin{document}

\title{Topology preserving thinning of cell complexes}

\author{Pawe{\l}~D{\l}otko,
        and~Ruben~Specogna}
\thanks{P. D{\l}otko is with the Department of Mathematics, University of Pennsylvania, Philadelphia, USA and Institute of Computer Science, Jagiellonian University, Krak\'{o}w, Poland e-mail: dlotko@sas.upenn.edu. P.D. is supported by the grant DARPA: FA9550-12-1-0416 and AFOSR: FA9550-14-1-0012.}
\thanks{R. Specogna is with the Department of Electrical, Mechanical and Management Engineering, Universit\`{a} di Udine, Udine, Italy email: ruben.specogna@uniud.it (see http://www.comphys.com). R.S. is partially supported by the Italian Ministry of Education, University and Research (MIUR) project PRIN 2009LTRYRE.}

\maketitle

\begin{abstract}
A topology preserving skeleton is a synthetic representation of an object that retains its topology and many of its significant morphological properties. The process of obtaining the skeleton, referred to as skeletonization or thinning, is a very active research area. It plays a central role in reducing the amount of information to be processed during image analysis and visualization, computer-aided diagnosis or by pattern recognition algorithms.

This paper introduces a novel topology preserving thinning algorithm which removes \textit{simple cells}---a generalization of simple points---of a given cell complex. The test for simple cells is based on \textit{acyclicity tables} automatically produced in advance with homology computations. Using acyclicity tables render the implementation of thinning algorithms straightforward. Moreover, the fact that tables are automatically filled for all possible configurations allows to rigorously prove the generality of the algorithm and to obtain fool-proof implementations. The novel approach enables, for the first time, according to our knowledge, to thin a general unstructured simplicial complex.
Acyclicity tables for cubical and simplicial complexes and an open source implementation of the thinning algorithm are provided as additional material to allow their immediate use in the vast number of practical applications arising in medical imaging and beyond.
\end{abstract}

\section{Introduction}

Thinning (or skeletonization) is the process of reducing an object to its skeleton.
The topology preserving skeleton may be informally defined as a thinned subset of the object that retains the same topology of the original object and often many of its significant morphological properties.
Thinning is a very active research area thanks to its ability of reducing the amount of information to be processed for example in medical image analysis and visualization as well as simplifying the development of pattern recognition or computer-aided diagnosis algorithms. Hence, it is not surprising that thinning gained a pivotal role in a wide range of applications.
An exhaustive review of the literature is beyond the scope of this work.
Two dimensional skeletons have been used for digital image analysis and processing, optical character and fingerprint recognition, pattern recognition and matching, and binary image compression since a long time ago, see for example the survey paper~\cite{LamLeeSuen}.
More recently, three dimensional skeletons have been widely used in computer vision and shape analysis~\cite{sundar}, in computer graphics for mesh animation~\cite{cornea} and in computer aided design (CAD) for model analysis and simplification~\cite{deyacm},~\cite{deycad} and for topology repair~\cite{topologyrepair}.

There is also a vast literature of applications of skeletons in medical imaging. They have been used for route planning in virtual endoscopic navigation~\cite{pathplanning}, for example in virtual colonoscopy~\cite{colon1}-\cite{tmi_huang} or bronchoscopy~\cite{bronchi}.
Skeletons have also been an important part of clinical image analysis by providing centerlines of tubular structures.
In particular, there is a large body of literature showing applications of skeletons to blood veins centerline extraction from angiographic images~\cite{tmi_huang},~\cite{tmi_angio}-\cite{coronary}, and intrathoracic airway trees classification for the evaluation of the bronchial tree structure~\cite{tmi_pal1}.
Also protein backbone models can be produced with techniques based on skeletons~\cite{proteins}.
Furthermore, many computer-aided diagnostic tools rely on skeletons. For example, skeletons have been used to identify blood vessels stenoses~\cite{tmi_kitamura}-\cite{tannen}, tracheal stenoses~\cite{tmi_pal2}, polyps and cancer in colon~\cite{polipi} and left atrium fibrosis~\cite{tmi_ravanelli}.

There are medical applications of skeletons where topology preservation is essential. Non invasively determine the three-dimensional topological network of the trabecular bone~\cite{tmi_gomberg} is a good example. Indeed, many studies demonstrate that the elastic modulus and strength of the bones is determined by the topological interconnections of the bone structure rather than the bone volume fraction~\cite{bones},~\cite{sahabones}. Therefore, topological analysis plays a fundamental role in computer-aided diagnostic tools for osteoporosis~\cite{bones}.

Topology preserving thinning is non trivial and a vast literature, briefly surveyed in Sec. \ref{prior}, has been dedicated to this topic.
In particular, thinning by iteratively removing \textit{simple points}~\cite{simplePoints} is a widely used and effective technique.
It works locally and for this reason is efficient and easy to implement.

While reading the literature one may notice that thinning algorithms are claimed to be ``topology preserving,'' even though in most cases a precise statement of what that means is left unaddressed. This paper uses \textit{homology theory}~\cite{hatcher} to rigorously define what the virtue of being topology preserving actually consists of. This theory is less intuitive than the concepts used so far, including simple homotopy type \cite{cohen}, but exhibits some important theoretical and practical advantages that will be highlighted later in the paper.
We remark that a homological definition of simple points has already been used in the context of skeletonization in~\cite{mish-Pil},~\cite{tannenbaum}, but only in the case of cubical complexes. This paper generalizes this idea to cell complexes that are more general than cubical complexes. There are many applications that would benefit from an algorithm that deals with general unstructured simplicial complexes~[p. 35]\cite{mesh}. In fact, the geometry of three-dimensional objects is frequently specified by a triangulated surface, obtained for example by using an isosurface algorithm as \textit{marching cubes}~[p. 539]\cite{mesh}, \cite{tmi_huang} applied on voxel data from computed tomography, magnetic resonance imaging or any other three-dimensional imaging technique. Another possibility is to obtain the triangulations from the convex hull of point clouds provided for example by 3d laser scanners.
Triangulated surfaces offer two potential advantages over voxel representation. They allow to adaptively simplify the surface triangulation, see for example Fig. 16.20 in~[p. 549]\cite{mesh}. They also allow to visualize and edit the object efficiently with off-the-shelf software (for example the many visualization and editing tools for stereo lithography) and without the starcase artifacts typical of voxel representation of objects with curved boundary. One may even easily print the object with additive manufacturing technology (i.e. 3d printers).

Another issue that arises reading the literature is that many different definitions of topology preserving skeleton exist.
In some papers, the skeleton is obtained by removing simple pairs in the spirit of simple homotopy theory by what is well known as \textit{collapsing} in algebraic topology \cite{hatcher}. The resulting skeleton, if no other constraints are used, has a lower dimension with respect to the input complex. On the contrary, this paper assumes that the skeleton is always a solid object of the same dimension as the initial complex.
The difference is highlighted in Fig.~\ref{fig:def}.
\begin{figure}[t]
\centerline{\includegraphics[scale=0.5]{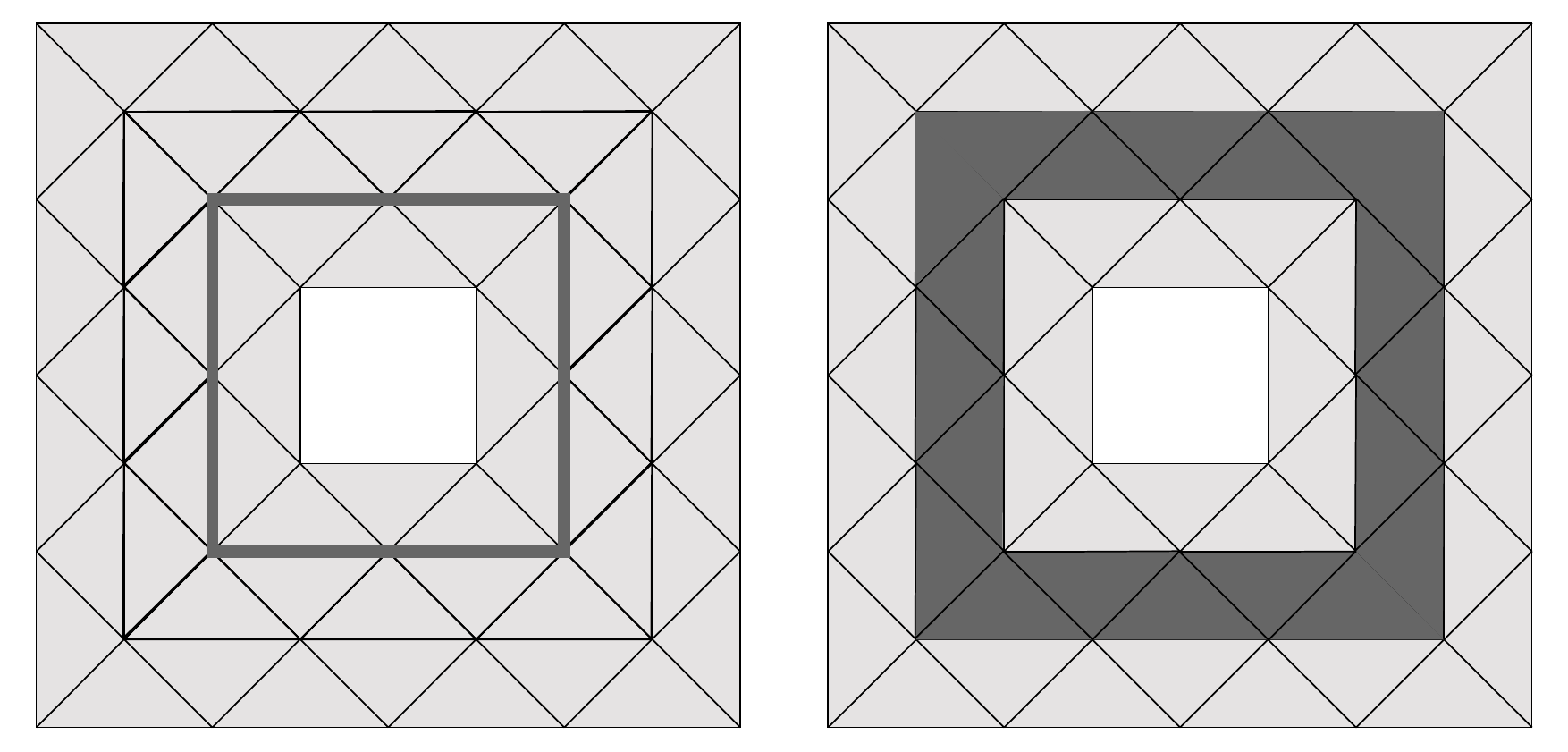}}
\caption{Let us consider, as an example, a 2-dimensional simplicial complex $\mathcal{K}$ representing an annulus.
On the left, the thick cycle represents a 1-dimensional skeleton of $\mathcal{K}$ obtained by means of standard \textit{collapsing} of $\mathcal{K}$ \cite{hatcher}. On the right, the gray triangles represent a 2-dimensional skeleton of $\mathcal{K}$ according to the definition used in this paper. This kind of skeleton is obtained after removing a sequence of top dimensional cells.}
\label{fig:def}
\end{figure}

In this paper the skeleton of a given complex $\mathcal{K}$ is defined as a subset $S \subset \mathcal{K}$ that is obtained from $\mathcal{K}$ after removing a sequence of top dimensional cells. We require that the \textit{homology}~\cite{hatcher} of the initial complex $\mathcal{K}$ is preserved during this process. In particular, a top dimensional cell can be safely removed if this does not change the homology of the \textit{complement} of $S$. Fig.~\ref{fig:2intersection} provides an intuitive explanation why the last requirement is desirable. This additional requirement, to the best of our knowledge, is not documented in other papers. We call this cell a \textit{simple cell}, which is a generalization of the idea of \textit{simple points}~\cite{simplePoints} in digital topology.
\begin{figure}[t]
\centerline{\includegraphics[scale=0.5]{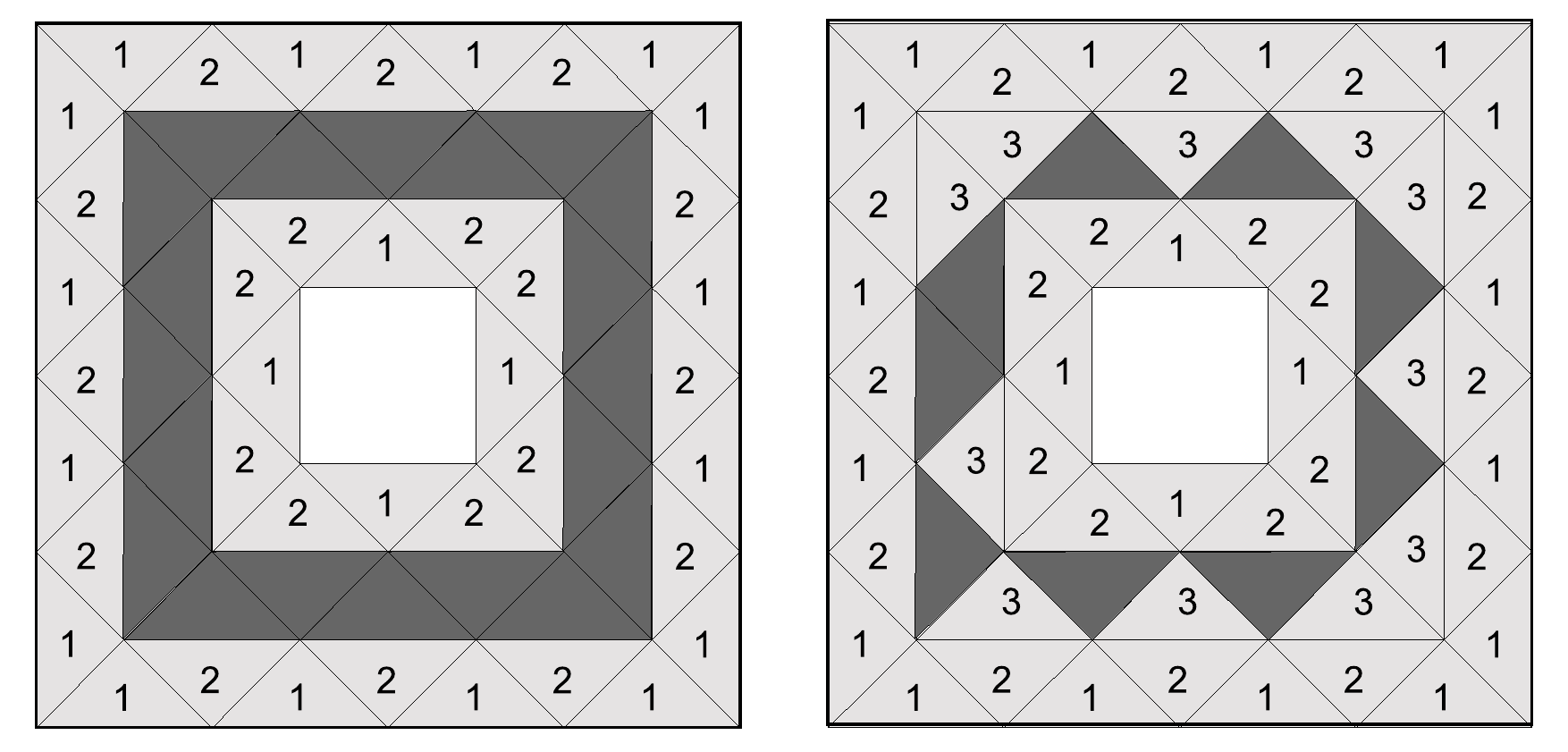}}
\caption{Suppose the iterative thinning Algorithm~\ref{alg:thinning} is used to skeletonize a 2-dimensional simplicial complex $\mathcal{K}$ representing an annulus. Let the dark gray triangles belong to the skeleton. On the left, the result obtained by checking whether the removal of a cell changes the topology of $\mathcal{K}$ complement. On the right, the result obtained by checking  whether the removal of a cell changes the topology of $\mathcal{K}$. The numbers inside triangles indicate the iteration number of the while loop in the thinning algorithm when they were removed.
Both skeletons preserve topology. However, in most applications, the skeleton on the left is preferred.
}
\label{fig:2intersection}
\end{figure}
Clearly, in nontrivial cases, the skeleton $S$ is not unique. 

Resorting to explicit homology computations to detect simple points as in~\cite{mish-Pil},~\cite{tannenbaum},~\cite{simplePoints} is quite computationally intensive, as the worst-case complexity of homology computations is cubical, see also the discussion in~\cite{simplePoints}. In this paper, we introduce a much more efficient solution by exploiting the idea of tabulated configurations, i.e. \textit{acyclicity tables}, that are described in detail in Sec.~\ref{sec:tables}.

Usually, a skeleton also requires to preserve the shape of the object.
In this paper we show some very simple proof of concept idea how to preserve both homology and shape. Of course, this is just an example to illustrate how the idea of acyclicity tables can be used together with some additional techniques that guarantee shape preservation.

The rest of the paper is organized as follows.
In Section~\ref{prior} the prior work on thinning algorithms is surveyed. Section~\ref{contrib} analyzes the original contributions of the present paper. In Section~\ref{sec:introToTopology} the property of being a homology preserving thinning is rigorously stated. In Section~\ref{sec:tables} the concept of acyclicity tables is introduced, whereas, in Section~\ref{sec:topologyPreservingAlgorithm}, the topology preserving thinning algorithm is presented. Section~\ref{sec:benchmarks} discusses the results of the thinning algorithm on a number of benchmarks and, finally, in Section~\ref{sec:conclusions} the conclusions are drawn.

\subsection{Prior work}\label{prior}
There are hundreds of papers about thinning.
Most of them fall into two categories. On one hand, there are papers using morphological operations like erosion and dilatations to obtains skeletons, see~\cite{kong} and references therein. They do not guarantee topology preservation in general. The others use the idea of removing the so called simple points from the given cell complex, see~\cite{mish-Pil},~\cite{tannenbaum},~\cite{simplePoints}.
Without pretending to be exhaustive, in the following we resume previous results. 

\subsubsection{2-dimensional images}
Most of the work on thinning regard finding skeletons of 2-dimensional images. A very comprehensive survey on this topic may be found in~\cite{LamLeeSuen}.
This case is well covered in literature and general solution exists, see for example \cite{zhang}-\cite{Ashwin}.

\subsubsection{3-dimensional cubical complexes}
In case one wants to skeletonize three (or higher) dimensional images, there are much less papers available in literature. Most of them rely on case study, see~\cite{lob}-\cite{palagi5}. The problem is that it is hard to prove that a rule-based algorithm is general, i.e. it removes a cell if and only if its removal does not change topology. In 3d there are more than 134 millions possible configurations for a cube neighborhood and only treating correctly \textit{all} of them gives a correct thinning algorithm. 
References~\cite{mish-Pil},~\cite{tannenbaum},~\cite{simplePoints} use explicit homology computations to detect simple points.

There are a number of papers presenting thinning algorithms for 3-dimensional images in which Euler characteristic is used to guarantee topology preservation, see for example~\cite{Ashwin}, \cite{abbera} and references therein. The problem is that Euler characteristic is a rather raw measure of topology and it is not sufficient to preserve topology in general for 3-dimensional cubical complexes. For three dimensional images one needs to use both Euler characteristic and connectivity information to preserve topology, but this is not sufficient for four dimensional images. 

\subsubsection{2-dimensional simplicial complexes and cell complexes}
All the strategies presented so far are applicable only to cubical grids (pixels, voxels, ...). To our best knowledge, there are just a few papers dealing with 2d grids that are not cubical, and they are restricted to 2d binary images modeled by a quadratic, triangular, or hexagonal cell complex, see~\cite{2dsim1}-\cite{2dsim4}.
The main reason for the lack of results on general 2d simplicial complexes may be the absence of regularity in unstructured simplicial grids that makes case-study algorithms very hard to devise and to implement. This gap in the literature is covered by the present paper.

\subsubsection{3-dimensional simplicial complexes and cell complexes}
To the best of our knowledge, we are not aware of algorithms that deal with unstructured 3d simplicial complexes or more general cell complexes. There are only some papers that find the 1-dimensional skeleton by using the well known collapsing in algebraic topology~\cite{skeletonsFromPointClouds}-\cite{liu}. 
Again, this gap in the literature is covered by the present paper.
%

\subsection{Summary of paper contributions}\label{contrib}

In this Section the main novelties presented in this paper are summarized:
\begin{enumerate}
\item The claim ``topology preserving thinning'' is rigorously defined, for any cell complex, by means of homology theory.
\item A novel topology preserving thinning algorithm that removes simple cells is introduced. Conceptually this algorithm falls into the category of thinning algorithms based on simple points and generalizes all previous papers. In fact, the acyclicity tables introduced in this paper give a classification of all possible simple points that can occur in a given cell complex. Therefore, no rules are needed since all of them are encoded into the acyclicity tables.
\item The most important advantage of the novel approach is that acyclicity tables are \textit{automatically} filled in advance, for any cellular decomposition, with homology computations performed by a computer. Therefore, once the tables are available, the implementation of a thinning algorithm is straightforward since identifying simple cells requires just queering the acyclicity table. No other topological processing is needed.
\item The fact that acyclicity tables are filled \textit{automatically} and \textit{correctly}, for \textit{all} possible configurations, provides a rigorous computer-assisted mathematical proof that the homology-based thinning algorithm preserves topology. It is also verified, simply by checking all acyclic configurations, that using Euler characteristic is not enough to ensure preservation of topology in 3-dimensional or higher dimensional cubical and simplicial complexes. However, when one checks both Euler characteristic and that the number of connected component before and after cell removal remains one, then topology is preserved. Checking Euler characteristic together with connectivity does not suffice to preserve topology in 4d.
\item The acyclicity tables for simplicial complexes of dimension 2, 3 and 4 and for cubical complexes of dimension 2 and 3, that can be freely used in any implementation of the proposed algorithm, are provided as supplemental material at~\cite{thinit}. This way, we dispense readers to implement homology computations to produce the acyclicity tables.
\item The thinning algorithm, unlike the standard collapsing of algebraic topology~\cite{hatcher}, does not require the whole cell complex data structure but it uses only the top dimensional elements of the complex, with obvious memory saving.
\item As a proof of concept, an open source C++ implementation that works for 3-dimensional simplicial complexes is provided to the reader as supplemental material at~\cite{thinit}. We remark that the code is optimized for readability and memory usage and not for speed.
\end{enumerate}


\section{Topology preserving thinning by preserving homology}
\label{sec:introToTopology}
When one claims that an algorithm ''preserves topology,'' in order to give a precise meaning to this statement, one needs to specify which topological invariant is preserved. In the literature, the invariant is assumed to be, in most cases implicitly, the so called homotopy type~\cite{hatcher}. The problem of this choice is that this strong topological invariant in general is not computable according to Markov~\cite{markov}. This is the reason why in this paper we propose to use homology theory which is computable in place of homotopy theory, even if it is weaker than the former. Indeed, homology seems to be the strongest topological invariant that can be rigorously and efficiently computed.
Therefore, every time we claim that topology is not changed, implicitly we mean that the homology is not changed.

Homology groups may be used to measure and locate holes in a given space. Zero dimensional holes are the connected components. One dimensional holes are handles of a given space, whereas two dimensional holes are voids totally surrounded by the considered space (i.e. cavities). One can look at a $n$-dimensional hole as something bounded by a deformed $n$-sphere. A space is homologically trivial (or \textit{acyclic}) if it has one connected component and no holes of higher dimensions. A rigorous definition of homology groups is not presented in this paper due to the availability of rigorous mathematical introductions in any textbook of algebraic topology as~\cite{hatcher} and the lack of space. For a more intuitive presentation for non mathematicians one one may consult~\cite{cpc},~\cite{cicp}.

In this paper, we consider in particular two standard ways of representing spaces, namely the \textit{simplicial} and \textit{cubical} complexes.
A \textit{n-simplex} is the convex hull of $n+1$ points in general position (point, edge, triangle, tetrahedron, 4-dimensional tetrahedron). A simplex spanned with vertices $x_1,\ldots,x_n$ is denoted by $[x_1,\ldots,x_n]$. By a \textit{face} of a $n$-simplex $A$ we mean the simplices spanned by a proper subset of vertices spanning $A$. A \textit{simplicial complex} $\mathcal{S}$ is a set of simplices such that for every simplex $A \in \mathcal{S}$ and every face $B$ of simplex $A$, $B \in \mathcal{S}$.
Pixels (2-cubes) and voxels (3-cubes) are widely used in image analysis. They form a Cartesian grid, that is a special case of grid where cells are unit squares or unit cubes and the vertices have integer coordinates. Even though we assume to deal with a Cartesian grid, the results presented in this paper hold also for more general grids such as a rectilinear grid, that is a tessellation of the space by rectangles or parallelepipeds that are not, in general, all congruent to each other. Therefore, we define a cubical complex $\mathcal{K}$ as a set of cubes such that for every cube $A \in \mathcal{K}$ and for every $B$ being face of we have $A$, $B \in \mathcal{K}$. We want to stress here that we assume every cell to be closed, i.e. if a cell is present in a complex, so do are its faces.

In the iterative thinning algorithm presented in this paper, the top dimensional cells (simplices, voxels) are iteratively removed from the object. Homology theory is used to ensure that removing of a given cell (simplex/cube) does not change the topology of the object. If removal of a cell does not change the topology, the cell is said to be \emph{simple}. Due to efficiency reasons, the homology cannot be recomputed after removing every single element. In fact, one may compute the homology of a cell complex for instance with~\cite{capd} software, but the worst case computational complexity is cubical. 
Therefore, the main idea is to rely on the so called Mayer--Vietoris sequence~\cite{hatcher}. Let us express the considered space $X$ as $X = X' \cup V$, where $V$ is a single top dimensional simplex or voxel. The sequence states that once the intersection $X' \cap V$ is homologically trivial, then the homology of $X'$ is the same as the homology of $X$. This important result is the key of the method presented in this paper. In fact, it implies that, in order to preserve the homology of $X$ with respect to $X'$, one should check the homology of the intersection $X' \cap V$. In practice, this may be easily performed with the~\cite{capd} software.

The main novelty of this paper is to present a different idea to speed up the computations.
Let $V$ be a simplex or voxel. By $bd\ V$ we denote the boundary of $V$, i.e. \textit{all} lower dimensional cells which are entirely contained in the closure of $V$.
The idea is based on the observation that in $bd\ V$ there are not too many elements, namely
\begin{enumerate}
\item $6$ in case of triangle ($2$-dimensional simplex);
\item $14$ in case of tetrahedron ($3$-dimensional simplex);
\item $30$ in case of 4-dimensional simplex (i.e. the convex hull of $5$ points in $\mathbb{R}^4$ in general position);
\item $8$ in case of pixel (2-dimensional cube);
\item $26$ in case of voxel (3-dimensional cube).
\end{enumerate}
By \emph{configuration} we mean any subset of $bd\ V$. When looking for simple cells, the configuration characterizes the way $V$ intersects the complement of the set that we aim to thin. A configuration is \emph{acyclic} if its homology---computed as the homology of the corresponding chain complex---is trivial. 
Since the number of all possible configurations in $bd\ V$ is $2^i$, where $i$ is the number of boundary elements of $V$, one may pre-compute the homology of all the configurations and store them in a lookup table. In this case, homology computations are done only in a pre-processing stage and once and for all. After creating such an acyclicity table, one may instantly (i.e. in $O(1)$ time) get the answer whether the intersection $X' \cap V$ is homologically trivial or not. This is the strategy that we aim to use in the thinning algorithm. Next section shows how to use the acyclicity tables and how to obtain them automatically.

\section{Use of acyclicity tables and their generation}\label{sec:tables}
Let us consider a generic cell $V$ of a cell complex. Let us fix an order of all boundary elements $\{b_1,\ldots,b_n\}$ of $V$. We consider all subsets of the set $\{b_1,\ldots,b_n\}$ and enumerate them in the following way. For $J \subset \{1,\ldots,n\}$ the number of a subset $\{b_i\}_{i \in J}$ is $\sum_{i \in J} 2^i$. The \emph{acyclicity table} of $V$ is an array of size $2^{n}$ having, at position $j = \sum_{i \in J} 2^i$, the value \emph{true} if the configuration $\{b_i\}_{i \in J}$ is acyclic and \emph{false} otherwise.
%

Let us describe how the acyclicity table is constructed and used starting by considering a tetrahedron (3-dimensional simplex), see Fig. \ref{fig:templates}a. Let us enumerate vertices, edges and faces of tetrahedron as in Fig. \ref{fig:templates}a.
%
%
\begin{figure}[!h]
\centerline{\includegraphics[width=\columnwidth]{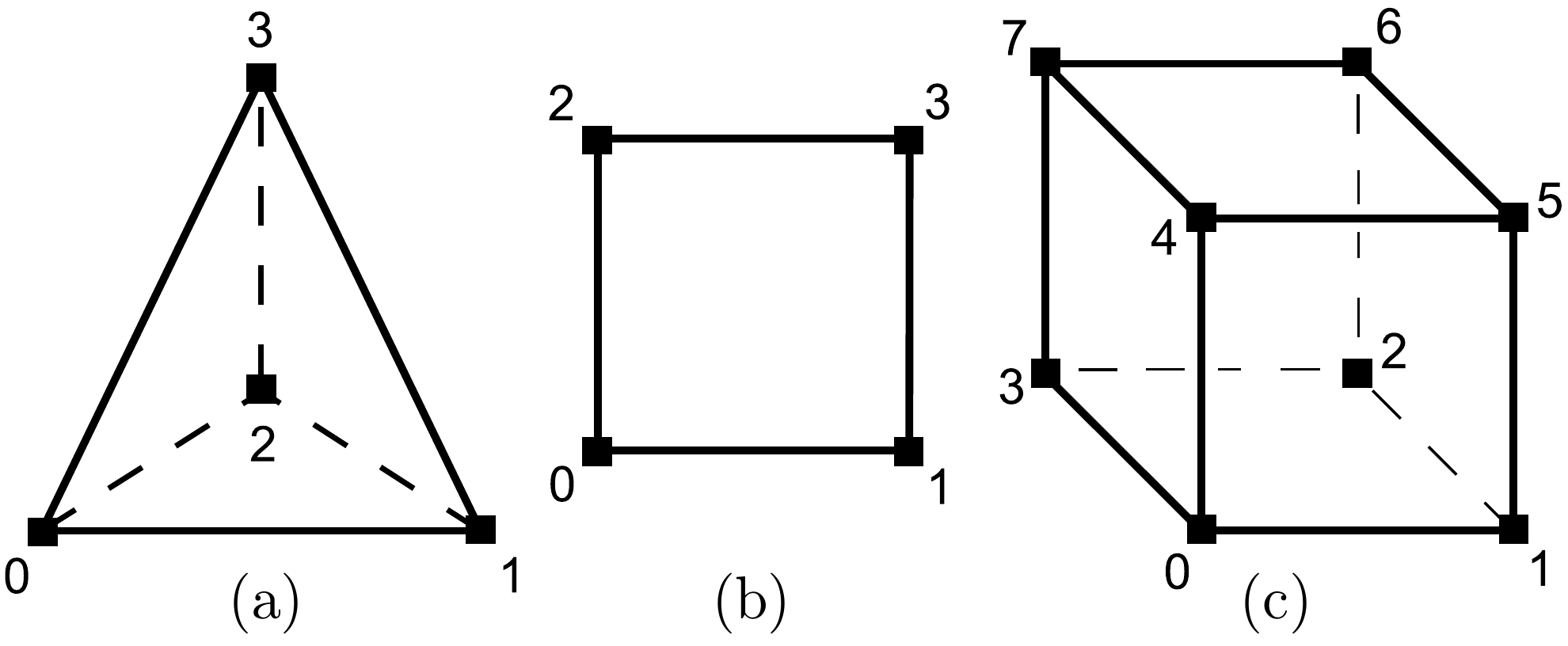}}
\caption{The model of the (a) $3$-dimensional simplex (tetrahedron of edges \texttt{01,02,03,12,13,23} and faces \texttt{012,013,023,123}), the (b) $2$-dimensional cube (pixel of edges \texttt{01,02,13,23}) the (c) $3$-dimensional cube (voxel of edges \texttt{01,03,04,12,15,23,26,37,45,47,56,67} and faces \texttt{0123,0145,0347,1256,2367,4567}). 
}
\label{fig:templates}
\end{figure}
We now introduce an ordering on boundary elements of the 3-dimensional simplex (bold numbers are the indexes of elements in the given order):
\begin{equation}
\begin{array}{ccccccc}
\textbf{1}&\textbf{2}&\textbf{3}&\textbf{4}&\textbf{5}&\textbf{6}&\textbf{7}\\
\texttt{0}&\texttt{1}&\texttt{2}&\texttt{3}&\texttt{01}&\texttt{02}&\texttt{03}\\
\textbf{8}&\textbf{9}&\textbf{10}&\textbf{11}&\textbf{12}&\textbf{13}&\textbf{14}\\
\texttt{12}&\texttt{13}&\texttt{23}&\texttt{012}&\texttt{013}&\texttt{023}&\texttt{123}.
\end{array}
\label{table:3simplex}
\end{equation}
%
Let $l_1,\ldots,l_k$ be the indexes of elements in the considered configuration (i.e. bold numbers corresponding to elements that are present in the configuration). The index of the configuration in the acyclicity table is computed with\[ index := \sum_{i=1}^k 2^{l_i}.\]

The acyclicity table is automatically generated in advance as follows. All possible configurations of the elements are automatically generated and the homology group of each of these configurations is computed with the~\cite{capd} software. If a configuration turns out to be acyclic, then a \emph{true} is set to the place in the array corresponding to the examined configuration, \emph{false} otherwise.

\begin{example}
Suppose the 3-dimensional simplex $[4,5,19,20]$ is given as input. Let $[4,5,19]$ and $[20]$ be the maximal elements in the configuration (i.e. the configuration consists of those elements and the vertices $[4],[5],[19]$, edges $[4,5]$, $[4,19]$, $[5,19]$ that are the faces of $[4,5,19]$). This configuration needs to be mapped into the 3-dimensional simplex model presented in Fig.~\ref{fig:templates}a. Hence, we have the following mapping between vertices: $4 \rightarrow$ \texttt{0}, $5 \rightarrow$ \texttt{1}, $19 \rightarrow$ \texttt{2}, $20 \rightarrow$ \texttt{3}. It is naturally extended to the mapping on simplices. Namely, the triangle $[4,5,19]$ is mapped to $[$\texttt{0},\texttt{1},\texttt{2}$]$ in the 3-dimensional simplex model, whereas vertex $[20]$ is mapped to vertex $[$\texttt{3}$]$. Therefore, the elements in this configuration are:
\begin{enumerate}
\item Vertices: $[$\texttt{0}$]$, $[$\texttt{1}$]$, $[$\texttt{2}$]$, $[$\texttt{3}$]$ (indices 1, 2, 3, 4);
\item Edges: $[$\texttt{0},\texttt{1}$]$, $[$\texttt{0},\texttt{2}$]$, $[$\texttt{1},\texttt{2}$]$ (indices 5, 6, 8);
\item Face: $[$\texttt{0},\texttt{1},\texttt{2}$]$ (index 11).
\end{enumerate}
Consequently, the index of this configuration is
\[ index = 2^1 + 2^2 + 2^3 +2^4 + 2^5 + 2^6 + 2^8 + 2^{11} =  2430.\]
One may check, at this position of the provided acyclicity table for 3-dimensional simplices, that this configuration is not acyclic\footnote{Clearly it cannot be, since it has two connected components.}.
\end{example}

In the same spirit, we introduce an ordering for the 2-dimensional simplex
\begin{equation}
\begin{array}{cccccc}
\textbf{1}&\textbf{2}&\textbf{3}&\textbf{4}&\textbf{5}&\textbf{6}\\
\texttt{0}&\texttt{1}&\texttt{2}&\texttt{01}&\texttt{02}&\texttt{12},
\end{array}
\label{table:2simplex}
\end{equation}
and for the 4-dimensional simplex
\begin{equation}
\begin{array}{cccccc}
\textbf{1}&\textbf{2}&\textbf{3}&\textbf{4}&\textbf{5}&\textbf{6}\\
\texttt{0}&\texttt{1}&\texttt{2}&\texttt{3}&\texttt{4}&\texttt{01}\\
\textbf{7}&\textbf{8}&\textbf{9}&\textbf{10}&\textbf{11}&\textbf{12}\\
\texttt{02}&\texttt{03}&\texttt{04}&\texttt{12}&\texttt{13}&\texttt{14}\\
\textbf{13}&\textbf{14}&\textbf{15}&\textbf{16}&\textbf{17}&\textbf{18}\\
\texttt{23}&\texttt{24}&\texttt{34}&\texttt{012}&\texttt{013}&\texttt{014}\\
\textbf{19}&\textbf{20}&\textbf{21}&\textbf{22}&\textbf{23}&\textbf{24}\\
\texttt{023}&\texttt{024}&\texttt{034}&\texttt{123}&\texttt{124}&\texttt{134}\\
\textbf{25}&\textbf{26}&\textbf{27}&\textbf{28}&\textbf{29}&\textbf{30}\\
\texttt{234}&\texttt{0123}&\texttt{0124}&\texttt{0134}&\texttt{0234}&\texttt{1234}.\\
\end{array}
\label{table:4simplex}
\end{equation}

In the case of cubes, unlike the case of simplices, the model cube is expressly needed to specify the location of vertices in the cube\footnote{This happens because in a cube not all vertices are connected with edges as in case of simplices. Therefore, the model cube is needed to point out the incidences of the vertices.}. The model for $2$- and $3$-dimensional cubes is represented in Fig.s \ref{fig:templates}b and \ref{fig:templates}c. 
%
The ordering for the 2-dimensional cube is
\begin{equation}
\begin{array}{cccccccc}
\textbf{1}&\textbf{2}&\textbf{3}&\textbf{4}&\textbf{5}&\textbf{6}&\textbf{7}&\textbf{8}\\
\texttt{0}&\texttt{1}&\texttt{2}&\texttt{3}&\texttt{01}&\texttt{02}&\texttt{13}&\texttt{23},\\
\end{array}
\label{table:2cube}
\end{equation}
%
%
whereas, for the 3-dimensional cube (voxel) is
\begin{equation}
\begin{array}{cccccc}
\textbf{1}&\textbf{2}&\textbf{3}&\textbf{4}&\textbf{5}&\textbf{6}\\
\texttt{0}&\texttt{1}&\texttt{2}&\texttt{3}&\texttt{4}&\texttt{5}\\
\textbf{7}&\textbf{8}&\textbf{9}&\textbf{10}&\textbf{11}&\textbf{12}\\
\texttt{6}&\texttt{7}&\texttt{01}&\texttt{03}&\texttt{04}&\texttt{12}\\
\textbf{13}&\textbf{14}&\textbf{15}&\textbf{16}&\textbf{17}&\textbf{18}\\
\texttt{15}&\texttt{23}&\texttt{26}&\texttt{37}&\texttt{45}&\texttt{47}\\
\textbf{19}&\textbf{20}&\textbf{21}&\textbf{22}&\textbf{23}&\textbf{24}\\
\texttt{56}&\texttt{67}&\texttt{0123}&\texttt{0145}&\texttt{0347}&\texttt{1256}\\
\textbf{25}&\textbf{26}&&&&\\
\texttt{2367}&\texttt{4567}.&&&&\\
\end{array}
\label{table:3cube}
\end{equation}
Of course, in order to compute the index in the acyclicity table, exactly the same procedure as the one described for the 3-dimensional simplex is used.

Historically, the acyclicitiy tables for cubes~\cite{acc} and simplices~\cite{ctic} were introduced in order to speed up homology computations. 
In this paper we provide an even stronger result. Not only the homology of the initial set and its skeleton is the same, but one can construct a retraction from the initial set to its skeleton. The existence of retraction implies the isomorphism in homology, but the existence of retraction is a stronger property than homology preservation. We demonstrated the existence of a retraction by a brute-force computer assisted proof, i.e. checking all acyclic configurations. 
Thus, the following lemma holds.
\begin{lemma}
For every acyclic configuration $C$ in the boundary of 2-, 3- or 4-dimensional simplices and 2- or 3-dimensional cubes (denoted as $bd(K)$) there exist a simple homotopy retraction from $bd(K) \setminus C$ to $C$.
\end{lemma}

At the end of this section, let us define more rigorously a \emph{simple cell}. 
\begin{definition}
A cell $T$ in a complex $\mathcal{K}$ is \textit{simple} if $(bd\ T\ \setminus (T \cap (\mathcal{K} \setminus T) ))$ is acyclic.
\end{definition}

In the supplemental material, we already provide the acyclicity tables for $2, 3, 4$ dimensional simplices and $2, 3$ dimensional cubes (pixels, voxels), in such a way that the reader can safely bypass the step of constructing them. We note that we do not provide tables for higher dimensional simplices or cubes, since the memory required to store them is huge\footnote{All configurations for $4$-dimensional cube require almost $10^9$ PB. All configurations for $5$-dimensional simplex require $4096$ PB. On the contrary, the acyclicity table for the 3-dimensional simplices provided as supplemental material requires no more than $32$kB.}.

\section{Topology preserving thinning algorithm}
\label{sec:topologyPreservingAlgorithm}
In this section we propose a simple thinning technique that iteratively removes simple cells. The algorithm is valid both for cubes and simplices provided that the corresponding acyclicity table is used. We want to point out that the algorithm works on top dimensional cells (cubes, simplices). Therefore---unlike the case of homological algorithms or collapsing---there is no need to generate the whole lower dimensional cell complex data structure\footnote{This structure have to be generated only locally for the boundary of a cell $T$ when checking if $T$ is simple.}. The input of the algorithm consists of a list $\mathcal{K}$ of top dimensional cells in the considered set. The output is a subset of $\mathcal{K}$ being its skeleton.

At the beginning, we present a first version of the algorithm that preserves only the topology of $\mathcal{K}$. 
%
At the beginning, one searches the list $\mathcal{K}$ to find all the cells $K_1,\ldots,K_n$ that are simple and store them into a queue $L$. Then, the queue $L$ is processed as long as it is not void. In each iteration, an element $K$ is removed from the queue $L$. Then, with the acyclicity table, one has to check if $K$ is simple in the set $S(\mathcal{K})$. We want to point out that elements already removed form the considered set $S(\mathcal{K})$ in previous iterations are treated as the exterior of $S(\mathcal{K})$ at a given iteration. If $K$ is simple in $S(\mathcal{K})$, then it is removed from the set $S(\mathcal{K})$. In this case, all neighbors of $K$\footnote{A neighbor of cell/simplex $K$ is any cell/simplex $K_1 \in \mathcal{K}$ such that $K \cap K_1 \neq \emptyset$.} that are still in $S(\mathcal{K})$ are added to the queue $L$.
The details of the presented procedure are formalized in Alg.~\ref{alg:thinning}.
\begin{algorithm}[!ht]
  \small
  \caption{Topology preserving thinning.}
  \label{alg:thinning}
  \begin{algorithmic}[1]
  \REQUIRE List of maximal cells $\mathcal{K}$;
  \ENSURE List of maximal cells $S(\mathcal{K})$ that belong to the skeleton of $\mathcal{K}$;
    \STATE Queue $L$;
    \STATE $S(\mathcal{K}) = \mathcal{K}$;
    \FOR{Every element $T \in S(\mathcal{K})$}
        \IF{$bd\ T\ \setminus (T \cap (S(\mathcal{K}) \setminus T) )$ is acyclic (check with acyclic table)}
			\STATE $L.enqueue(T);$\label{alg:first}
        \ENDIF
    \ENDFOR
    \WHILE{ $L \neq \emptyset$ }
		\STATE $T = L.dequeue();$
		\IF{$T \not \in S(\mathcal{K})$}
			\STATE Continue;
		\ENDIF
		\IF{($bd\ T\ \setminus (T \cap (S(\mathcal{K}) \setminus T) ))$ is acyclic}
			\STATE $S(\mathcal{K}) = S(\mathcal{K}) \setminus T$; \label{alg:remove}
			\STATE Put all neighbor cells of $T$ in $S(\mathcal{K})$ to the queue $L$;\label{alg:second}
		\ENDIF
    \ENDWHILE
    \RETURN $S(\mathcal{K})$;
  \end{algorithmic}
  \label{alg:matRed}
\end{algorithm}
We want to stress that Alg.~\ref{alg:thinning} is just an illustration. It may be turned into an efficient implementation by using more efficient data structures (for instance removing
 from the list $S(\mathcal{K})$ can be replaced by a suitable marking the considered element.) Also searching for intersection of $T$ with current $S(\mathcal{K})$ should be performed by using
 hash tables that, for the sake of clarity, are not used explicitly in Alg.~\ref{alg:thinning}. Let us now discuss the complexity of the algorithm. Clearly the \emph{for} loop
 requires $O(\cardd{\mathcal{K}})$ operations. We assume that one can set and check a flag of every cell in a constant time. This flag indicates if a cell is removed from $S(\mathcal{K})$ or not.
 Every cell $T \in \mathcal{K}$ appears in the \emph{while} loop only $k$ times, where $k$ is maximal number of neighbors of a top dimensional cell in the complex. Therefore, the \emph{while} loop
 performs at most $k \cardd{\mathcal{K}}$ iterations before its termination. The time complexity of every iteration is $O(k)$, which means that the overall complexity of the procedure is $O(k^2 \cardd{\mathcal{K}})$.
 Typically the number $k$ is a dimension dependent constant and, in this case, the complexity of the algorithm is $O(\cardd{\mathcal{K}})$. The same complexity analysis is valid for Alg.~\ref{alg:matRedShapePreserving}. 

We now present in Alg.~\ref{alg:matRedShapePreserving} a simple idea that enables to preserve the shape of the object in addition to its topology. 
We stress that the aim of this second algorithm is just to show how to couple topology and shape preservation.
\begin{algorithm}[!ht]
  \small
  \caption{Shape and topology preserving thinning.}
  \begin{algorithmic}[1]
  \REQUIRE List of maximal cells $\mathcal{K}$;
  \ENSURE List of maximal cells $S(\mathcal{K})$ that belong to the skeleton of $\mathcal{K}$;
    \STATE Queue $L$;
    \STATE $S(\mathcal{K}) = \mathcal{K}$;
    \FOR{Every element $T \in S(\mathcal{K})$}
        \IF{$bd\ T\ \setminus (T \cap (S(\mathcal{K}) \setminus T) )$ is acyclic (check with acyclic table)}
			\STATE $L.enqueue(T);$
        \ENDIF
    \ENDFOR
    \STATE Queue $K$;
    \WHILE{ $L \neq \emptyset$ }
		\STATE $T = L.dequeue();$
		\IF{$T \in S(\mathcal{K})$}
			\IF{($bd\ T\ \setminus (T \cap (S(\mathcal{K}) \setminus T) ))$ is acyclic}\label{acyclicity}
				\STATE $S(\mathcal{K}) = S(\mathcal{K}) \setminus T$;
			\ENDIF
            \STATE Put all the neighbor cells of $T$ in $S(\mathcal{K})$ to the queue $K$;
		\ENDIF
		\IF{ $L = \emptyset$ }
			\STATE $L = K$;
			\STATE $K = \emptyset$;
			\IF{all the cells in $S(\mathcal{K})$ have a top dimensional face in external boundary}\label{shapePreservingCodition}
				\STATE \emph{Break};
			\ENDIF
		\ENDIF
    \ENDWHILE
    \RETURN $S(\mathcal{K})$;
  \end{algorithmic}
  \label{alg:matRedShapePreserving}
\end{algorithm}
In Alg.~\ref{alg:matRedShapePreserving} there is one basic difference with respect to Alg.~\ref{alg:matRed}. In Alg.~\ref{alg:matRedShapePreserving}, after removing a single external layer of cells, a check is made at line~\ref{shapePreservingCodition} to determine whether all cells that remain in $S(\mathcal{K})$ are already in the boundary of $S(\mathcal{K})$. Once they are, the thinning process terminates. The topology is still preserved due to line~\ref{acyclicity}. 
The additional constraint used at line~\ref{shapePreservingCodition} of Alg.~\ref{alg:matRedShapePreserving} is very simple and it gives acceptable results in practice. It may be easily coupled with other techniques to preserve shape already described in literature.

Finally, we discuss the situation when one wants to keep the skeleton attached to some pieces of the external boundary $B$ of the mesh. In this case, when testing whether a top dimensional cell $T$ is simple, one should consider $B \cap bd\ T$ as elements in $S(K)$. In other words, elements from $B$ are not considered as an interface between the object to skeletonize and its exterior. 


\subsection{Proofs}
\label{sec:defOfSkeleton}

Now we are ready to give a formal definition of skeleton.
\begin{definition}
Let us have a simplicial or cubical complex $\mathcal{K}$. A \textit{skeleton} of $\mathcal{K}$, denoted by $S(\mathcal{K})$, is a set of top dimensional simplices or cubes such that:
\begin{enumerate}
\item $S(\mathcal{K})$ is obtained from $\mathcal{K}$ by iteratively removing top dimensional elements $T_1,\ldots,T_n$, provided that the intersection of $T_i$ with $\mathcal{K} \setminus \bigcup_{j=1}^{i-1} T_j$ complement is acyclic. Consequently, homology groups of $S(\mathcal{K})$ and $\mathcal{K}$ are isomorphic;
\item There is no top dimensional element $T \in S(\mathcal{K})$ that has an acyclic intersection with $S(\mathcal{K})$ complement (i.e. the process of removing such elements has been run as long as possible.)
\end{enumerate}
\label{def:skeleton}
\end{definition}
We want to point out that sometimes, due to some deep phenomena arising in simple homotopy theory, some skeleton may be redundant. For instance it is possible to have a skeleton of a 3-dimensional ball that is a Bing's house~\cite{bing} instead being a single top dimensional element. In general it is impossible to avoid this issue due to some intractable problems in topology.



In the follwing, we formally show that the skeleton obtained from Alg.~\ref{alg:thinning} satisfies Def.~\ref{def:skeleton}. This fact is shown with a sequence of two simple lemmas. 
\begin{lemma}
The homology of $\mathcal{K}$ and $S(\mathcal{K})$ are isomorphic.
\end{lemma}
\begin{proof}
The proof of this lemma is a direct consequence of the Mayer--Vietoris sequence~\cite{hatcher}. Let $T_1,\ldots,T_n$ be the elements removed during the course of the algorithm (enumeration is given by the order they were removed by the algorithm.)
Let us show that, for every $i \in \{1,\ldots,n\}$, homology of $\mathcal{K} \setminus \bigcup_{j=1}^{i-1} T_j$ and homology of $\mathcal{K} \setminus \bigcup_{j=1}^{i} T_j$\footnote{The difference used in the formulas in this proof is not a set theoretic difference. All the objects are assumed to contain all their faces.} are isomorphic. Let us write the Mayer--Vietoris sequence in reduced homology for $ \mathcal{K} \setminus \bigcup_{j=1}^{i-1} T_j = (\mathcal{K} \setminus \bigcup_{j=1}^{i} T_j) \cup T_i$:
\[ \ldots \rightarrow
\overline{H}_n( (\mathcal{K} \setminus \bigcup_{j=1}^{i} T_j) \cap T_i )
\rightarrow
\]
\[
\rightarrow
\overline{H}_n( \mathcal{K} \setminus \bigcup_{j=1}^{i} T_j )
\oplus
\overline{H}_n( T_i )
\rightarrow
\overline{H}_n( \mathcal{K} \setminus \bigcup_{j=1}^{i-1} T_j )
\rightarrow \ldots. \]
The intersection $(\mathcal{K} \setminus \bigcup_{j=1}^{i} T_j) \cap T_i$ is acyclic. This is because the intersection of $T_i$ with the set complement is checked in the acyclicity tables to be acyclic. Once it is, also $(\mathcal{K} \setminus \bigcup_{j=1}^{i} T_j) \cap T_i$ is acyclic. Therefore, $\overline{H}_n( (\mathcal{K} \setminus \bigcup_{j=1}^{i} T_j) \cap T_i )$ is trivial. Also, since $T_i$ is a simplex or cube, it is acyclic. This provides $\overline{H}_n(T_i)$ being trivial (we are considering reduced homology.) Consequently from the exactness of the presented sequence we have the desired isomorphism between $\overline{H}_n( \mathcal{K} \setminus \bigcup_{j=1}^{i} T_j )$ and $\overline{H}_n( \mathcal{K} \setminus \bigcup_{j=1}^{i-1} T_j ) $. The conclusion follows from a simple induction.
\end{proof}

\begin{lemma}
\label{lem:acIntersection}
After termination of the algorithm there is no element $T \in S(\mathcal{K})$ that has an acyclic intersection with the $S(\mathcal{K})$ complement.
\end{lemma}
\begin{proof}
Let $T_1,\ldots,T_n$ be the elements removed during the course of the algorithm (enumeration is given by the order they were removed by the algorithm.)
Suppose, by contrary, that a $T \in S(\mathcal{K})$ exists such that it has an acyclic intersection with the $S(\mathcal{K})$ complement. Let $i \in \{1,\ldots,n\}$ denotes the index of last element among $T_1,\ldots,T_n$ that has nonempty intersection with $T$.
If $i=0$, then $T$ would be put to the queue $L$ in the line~\ref{alg:first} of Alg.~\ref{alg:thinning} and removed from $S(\mathcal{K})$ in the line~\ref{alg:remove} of the algorithm, since no change to its intersection with $S(\mathcal{K})$ complements is made by removing $T_1,\ldots,T_n$.
If $i>0$, then after removing $T_i$ the intersection of $T$ with $S(\mathcal{K})$ complement does not change. Therefore, it is acyclic after removing $T_j$ for $j \leq i$. When Alg.~\ref{alg:thinning} removes $T_i$ in the line~\ref{alg:second}, $T$ is added to the list $L$ and it is going to be removed in the line~\ref{alg:remove}, since removing $T_j$ for $j > i$ does not affect the acyclicity of the intersection of $T$ with the $S(\mathcal{K})$ complement.
In both cases we showed that $T$ is removed from $S(\mathcal{K})$ by Alg.~\ref{alg:thinning}. Therefore, a contradiction is obtained.
\end{proof}



\section{Experimental results}
\label{sec:benchmarks}

\subsection{Assessing aortic coarctation and aneurism}
Skeletons can be used in computer-aided diagnostic tools for coarctation and aneurism, by evaluating the transverse areas of any vessel structure, see for example \cite{tmi_kitamura}.

\subsubsection{Aortic coarctation}
\begin{figure}[t]
\centering
\includegraphics[width=\columnwidth]{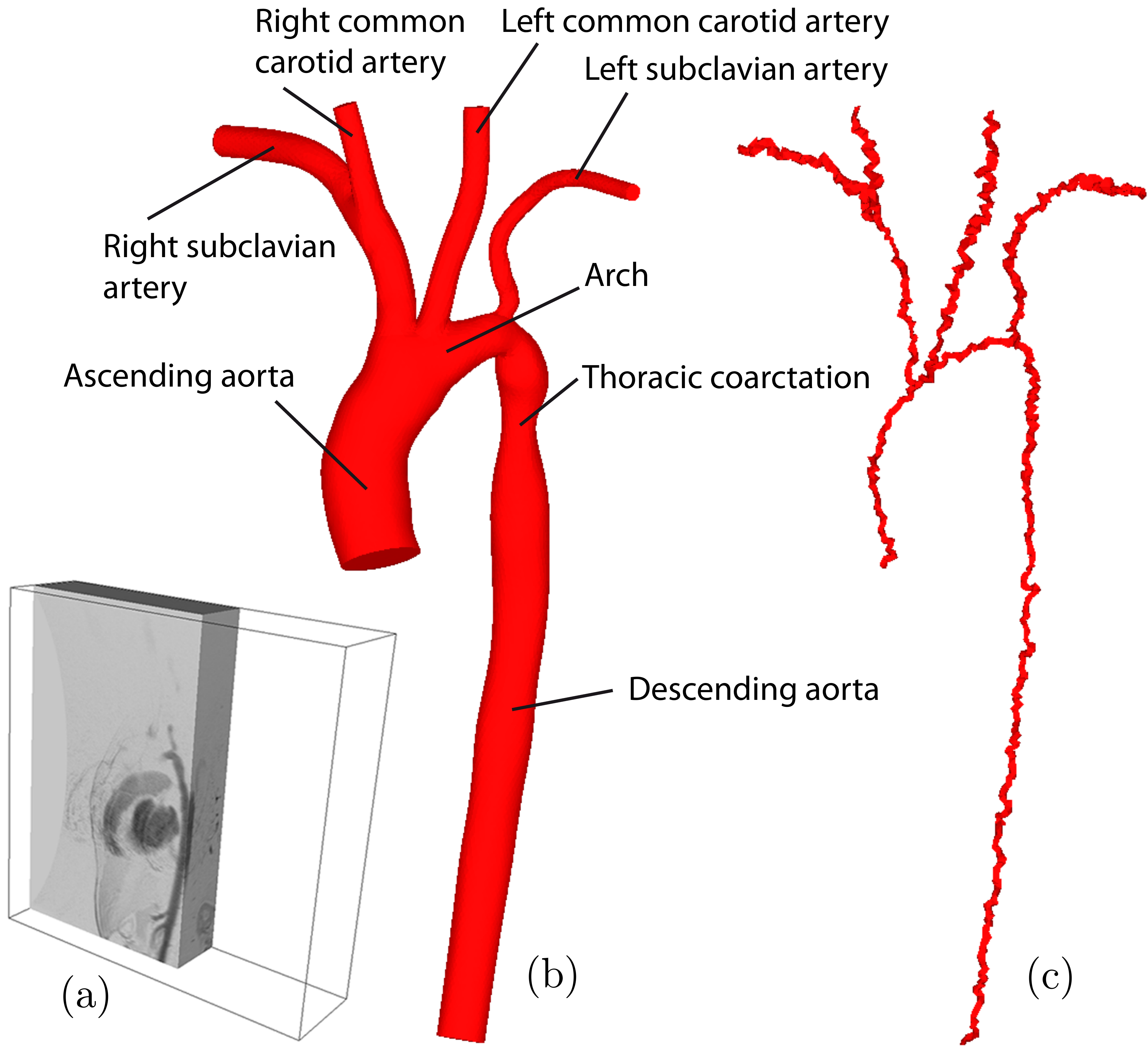}
\caption{(a) Magnetic resonance angiography (MRA) image of a moderate thoracic aortic coarctation. (b) Rendering of the 3D triangulated surface (20922 triangles) that represents the patient-specific thoracic aortic coarctation anatomy obtained by segmenting MRA data and (c) the skeleton extracted with Alg. 2.}\label{ex1}
\end{figure}
Aorta coarctation is a congenital heart defect consisting of a narrowing of a section of the aorta. Surgical or catheter-based treatments seek to alleviate the blood pressure gradient through the coarctation in order to reduce the workload on the heart. The pressure gradient is dependent on the anatomic severity of the coarctation, which can be determined from patient data.
Gadolinium-enhanced magnetic resonance angiography (MRA) has been used in a 8 year old female patient to image a moderate thoracic aortic coarctation, see Fig. \ref{ex1}a. Fig. \ref{ex1}b shows a rendering of the 3D triangulated surface, obtained by segmenting the MRA data, which models the ascending aorta, arch, descending aorta, and upper branch vessels. The interior of the surface has been covered with 94756 tetrahedra. The skeleton of this vessel structure, obtained with Alg. 2, is shown in Fig. \ref{ex1}c.


\subsubsection{Cerebrovascular aneurism}
\begin{figure}[t]
\centering
\includegraphics[width=\columnwidth]{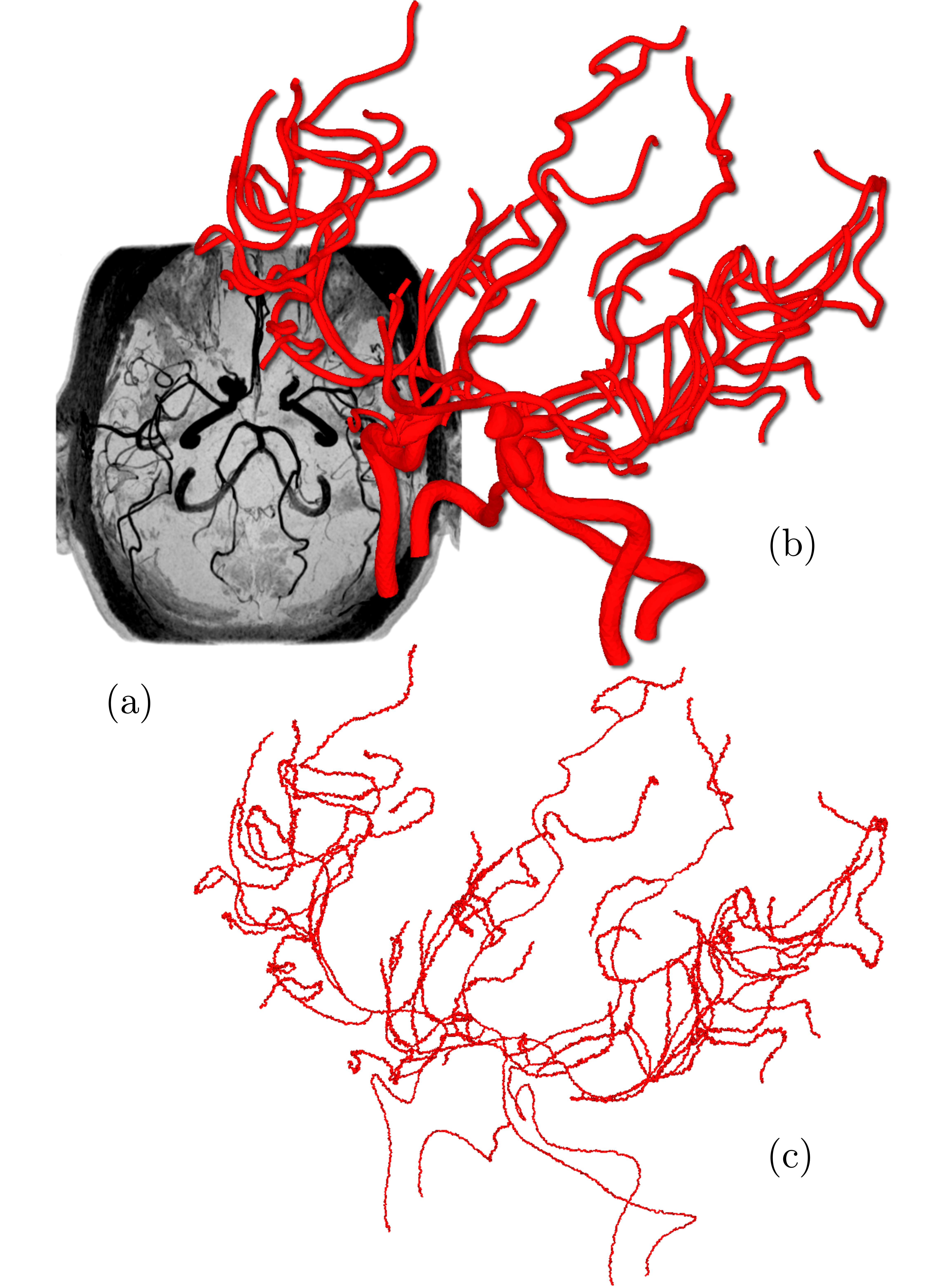}
\caption{(a) Magnetic resonance imaging (MRI) image of the cerebral circulation. (b) Rendering of the 3D triangulated surface (196,056 triangles) representing the cerebral circulation obtained by segmenting the MRI data and (c) the skeleton extracted with Alg. 2.}\label{brain}
\end{figure}
Cerebrovascular aneurysms are abnormal dilatations of an artery that supplies blood to the brain. Magnetic resonance imaging (MRI) has been used
to image the cerebral circulation in a 47 year old female patient, see Fig. \ref{brain}a. Fig. \ref{brain}b shows a rendering of the 3D triangulated surface, obtained from the segmentation of the MRI data. The interior of the surface has been covered with 390,081 tetrahedra. The skeleton of this vessel structure, obtained with Alg. 2, is shown in Fig. \ref{brain}c.

\subsection{Analysis of pulmonary airway trees}
\begin{figure}[t]
\centering
\includegraphics[width=\columnwidth]{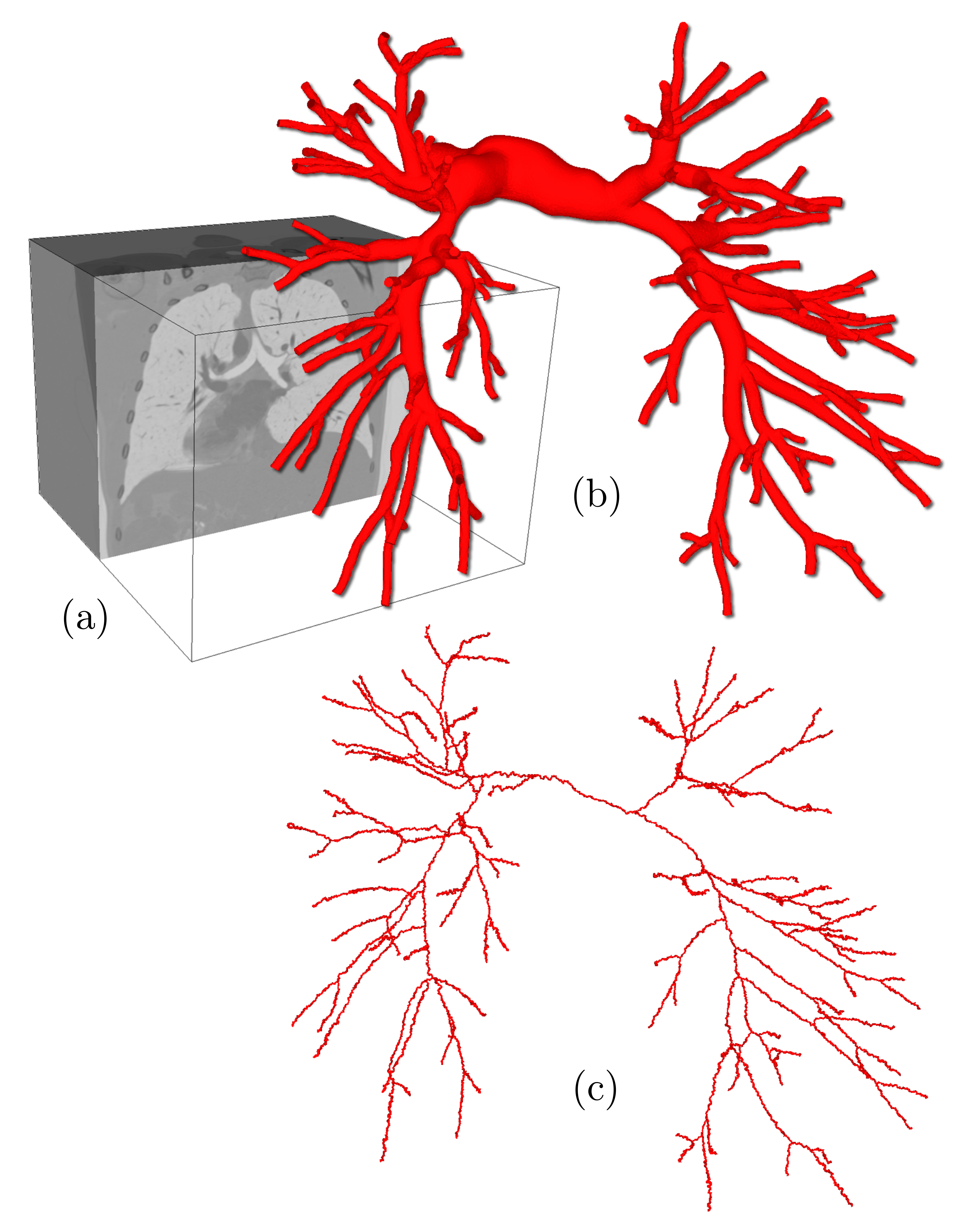}
\caption{(a) Computed tomography (CT) image of the pulmonary airway trees. (b) Rendering of the 3D triangulated surface (71926 triangles) representing the patient-specific pulmonary arteries and (c) the skeleton extracted with Alg. 2.}\label{pulmonary}
\end{figure}
Pulmonary arteries connect blood flow from the heart to the lungs in order to oxygenate blood before being pumped through the body.
Skeletons have been used for quantitative analysis of intrathoracic airway trees in \cite{tmi_pal1}.
A 3D triangulated surface, shown in Fig. \ref{pulmonary}b, represents the 3D model of pulmonary airway trees of a 16 year old male patient obtained by segmenting data from computed tomography (CT) images, see Fig. \ref{pulmonary}a. The interior of the surface is covered with 236,433 tetrahedra. The topology preserving skeleton obtained by Alg. 2 is shown in Fig. \ref{pulmonary}c.

\subsection{Extracting centerline for virtual colonoscopy}
\begin{figure}[t]
\centering
\includegraphics[width=6cm]{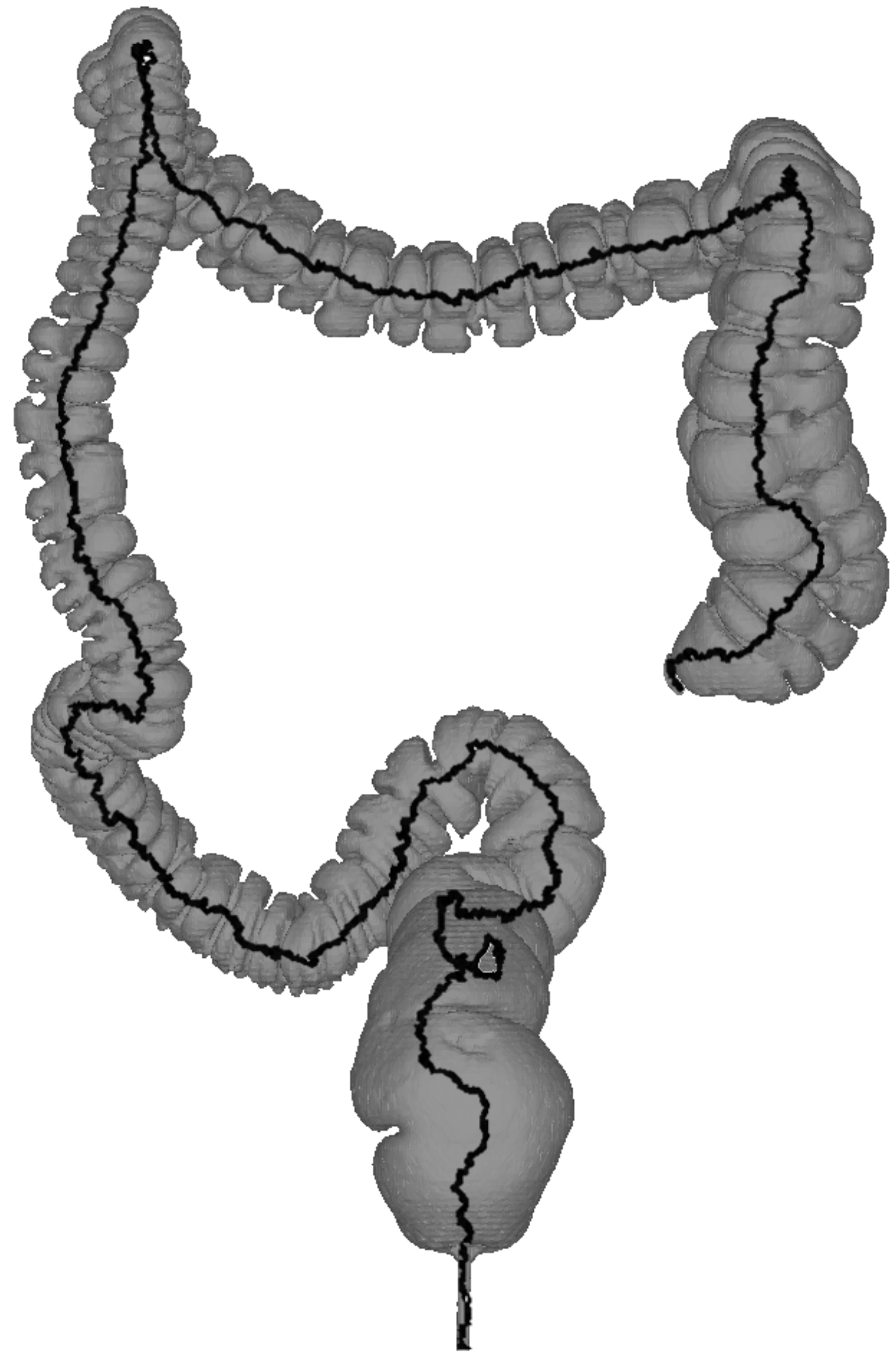}
\caption{Rendering of the 3D triangulated surface (506,188 triangles) representing the patient-specific colon anatomy obtained by segmenting CT colonography data and, in black, the skeleton extracted with Alg. 2.}\label{excolon}
\end{figure}
A 3D triangulated surface that represents the 3D model of a colon is obtained by segmenting data from computed tomography
(CT) images, see Fig. \ref{excolon}. The interior of the surface is covered with 2,108,424 tetrahedra. The topology preserving skeleton obtained by Alg. 2, which may be used as a colon centerline to guide a virtual colonoscopy, is shown in black in Fig. \ref{excolon}.

\subsection{Computing topological interconnections of bone structure}
\begin{figure}[t]
\centering
\includegraphics[width=\columnwidth]{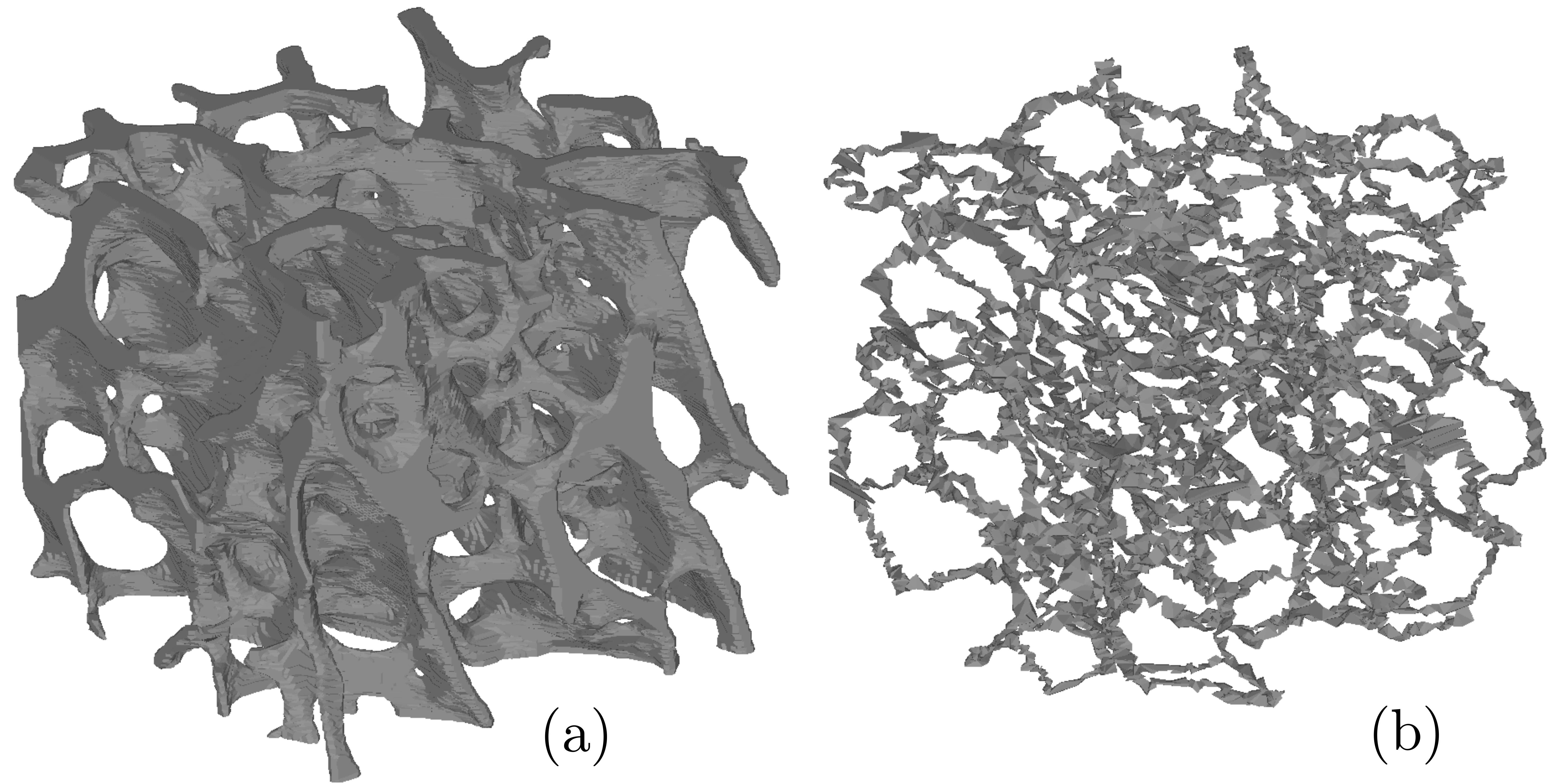}
\caption{(a) Rendering of the 3D triangulated surface (285,346 triangles) representing the anatomy of a human bone in a region of interest of the trabecular region obtained by segmenting MicroCT data and (b) the skeleton extracted with Alg. 1.}\label{exbone}
\end{figure}
A 3D model of a human bone belonging to a 61 year old male patient has been obtained from a stack of thresholded 2D images acquired by X-ray MicroCT scanning \cite{bonemodel}. In particular, a region of interest (ROI) of size $4\,$mm$ \times 4\,$mm ($200 \times 200$ pixels, pixel size $19.48\,\mu$m) is selected in the trabecular region. A stack of 195 2D images has been considered, resulting in a volume of interest (VOI) of approximately $4$mm$ \times 4$mm$ \times 4$mm. From this 3D model, consisting of about 2.4 millions voxels, a 3D triangulated surface has been obtained, see Fig. \ref{exbone}a. The interior of this surface is covered with 688,773 tetrahedra. The topology preserving skeleton obtained by Alg. 1 is shown in Fig. \ref{exbone}b.

\subsection{Some other non medical examples.}
The results of Algorithm \ref{alg:matRed} on some benchmarks are visible in Fig.s \ref{ex1}-\ref{ex14}, whereas the results obtained with Algorithm \ref{alg:matRedShapePreserving} are shown in Fig.s \ref{ex5}-\ref{ex13}.
\begin{figure}[!t]
\centering
\includegraphics[width=3.5in]{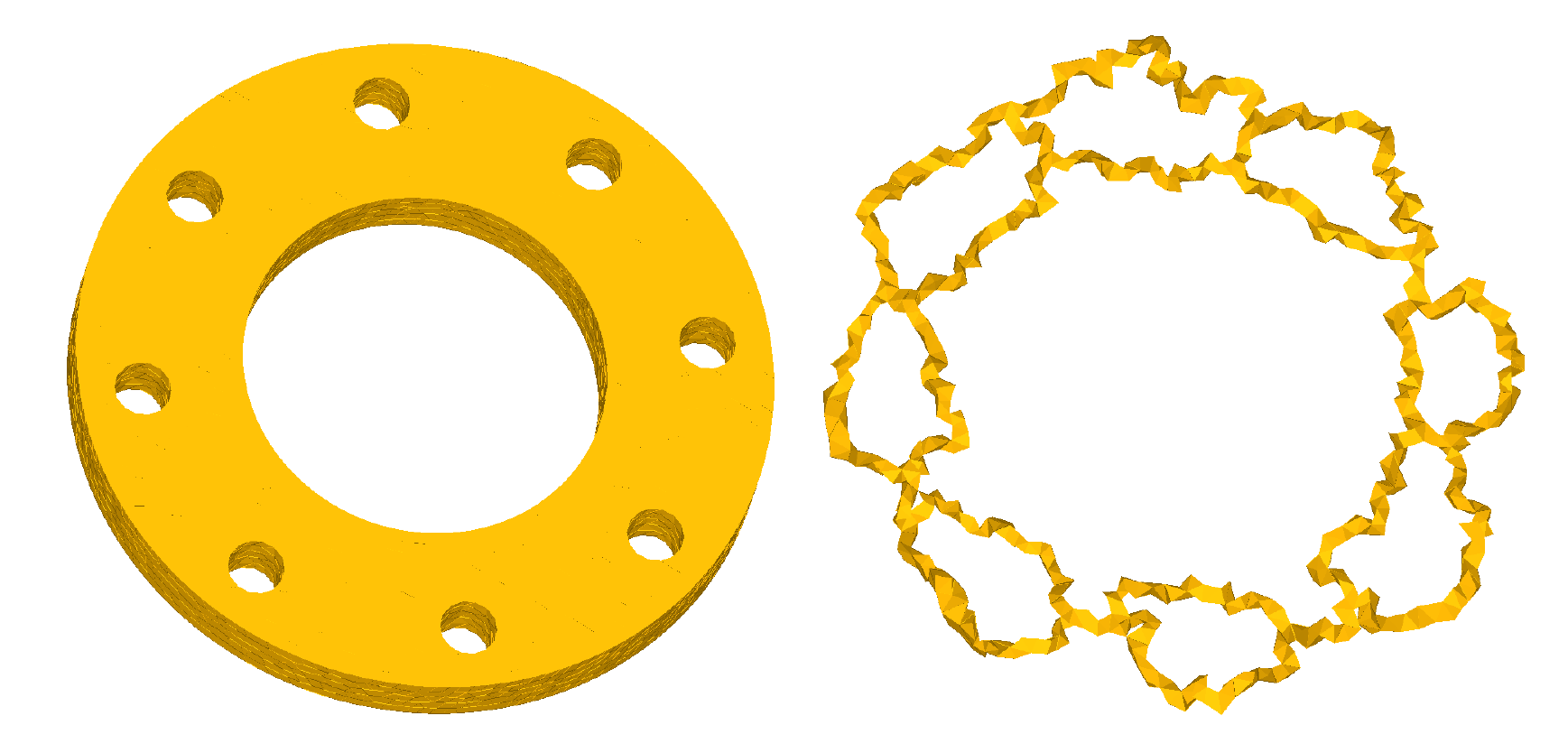}
\caption{A flange and its topology-preserving skeleton.}\label{ex1}
\end{figure}
\begin{figure}[!t]
\centering
\includegraphics[width=3.5in]{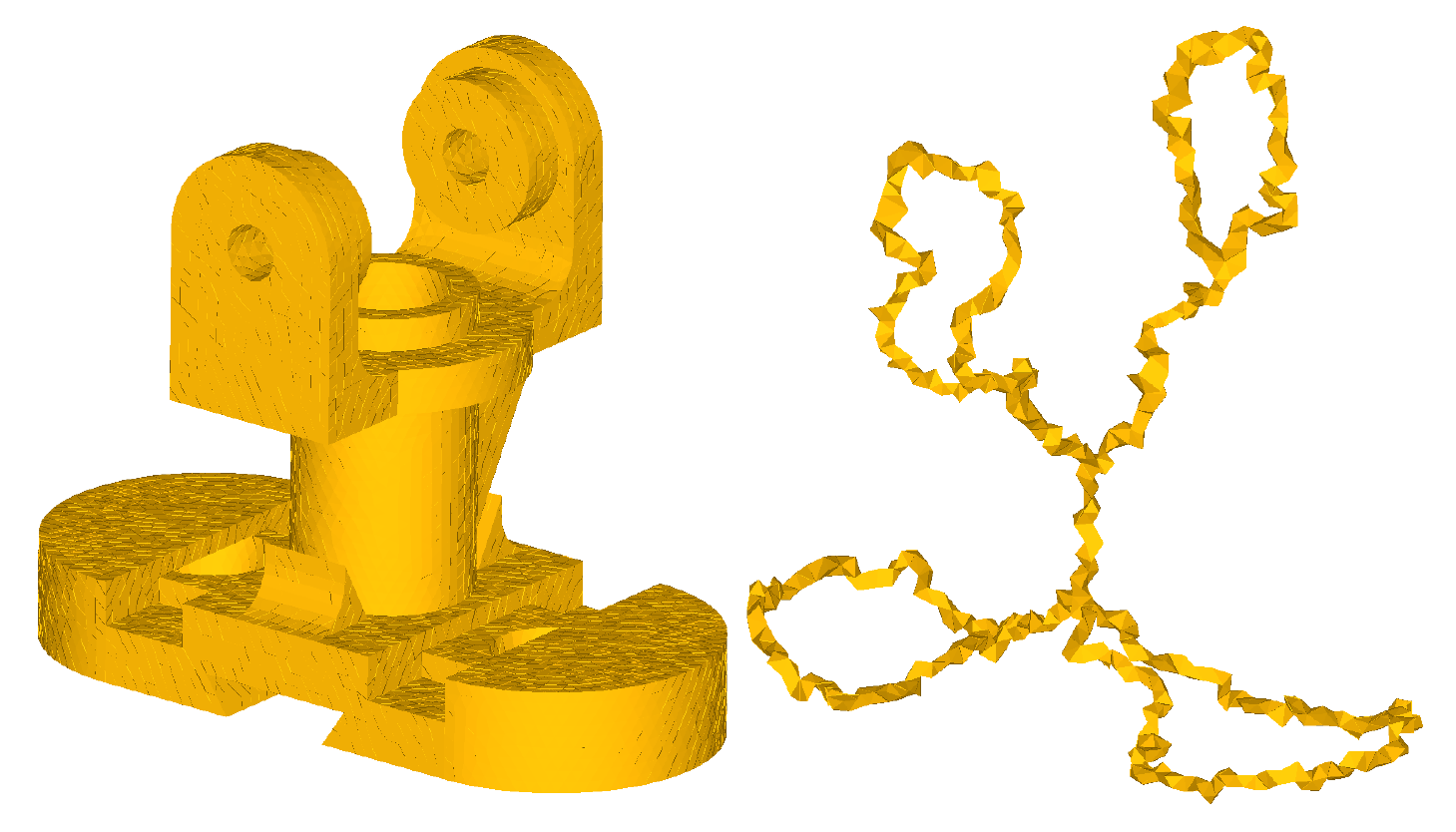}
\caption{A mechanical part and its topology-preserving skeleton.}\label{ex2}
\end{figure}
\begin{figure}[!t]
\centering
\includegraphics[width=3.5in]{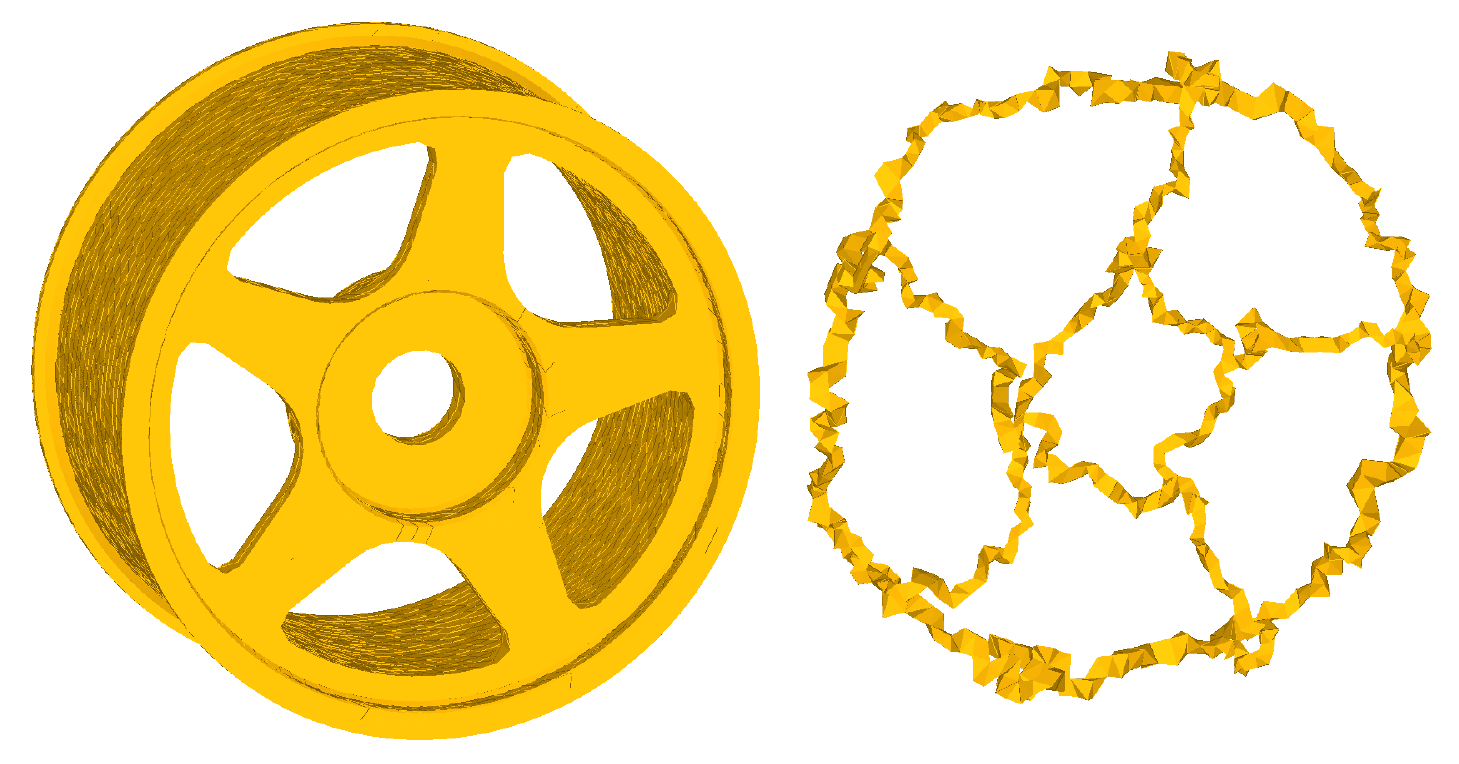}
\caption{A racing rim and its topology-preserving skeleton.}\label{ex3}
\end{figure}
\begin{figure}[!t]
\centering
\includegraphics[width=3.5in]{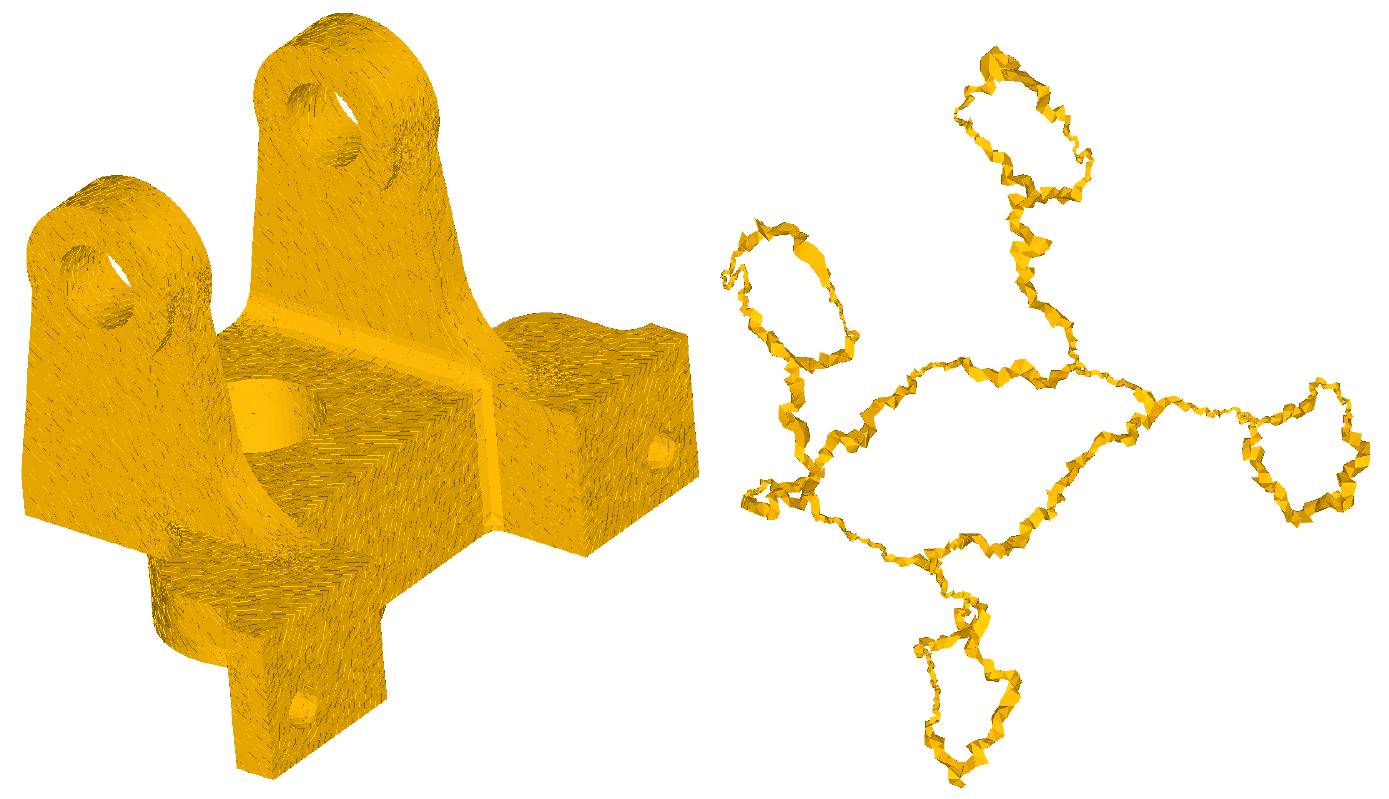}
\caption{A mechanical support and its topology-preserving skeleton.}\label{ex4}
\end{figure}
\begin{figure}[!t]
\centering
\includegraphics[width=3.5in]{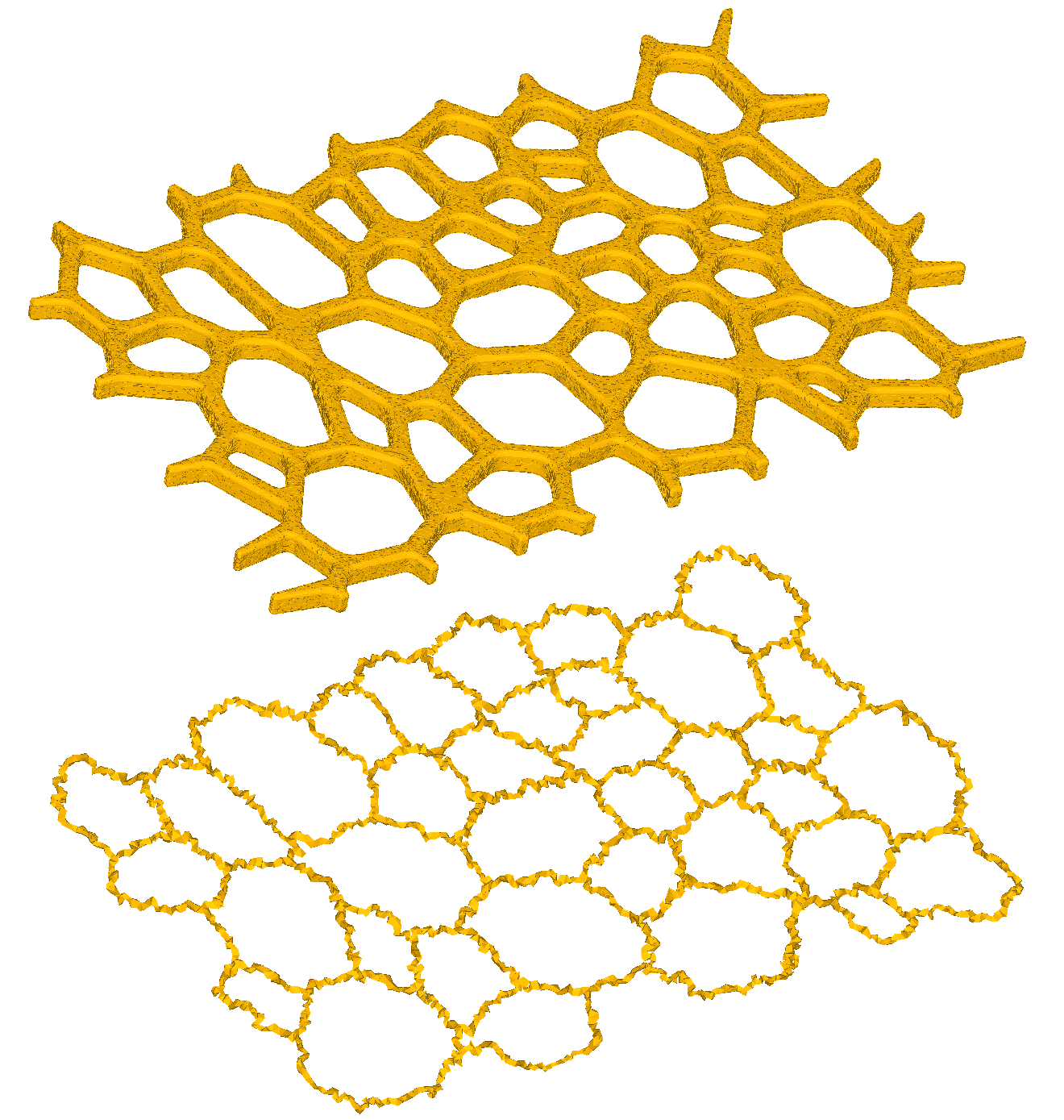}
\caption{A net and its topology-preserving skeleton.}\label{ex7}
\end{figure}
\begin{figure}[!t]
\centering
\includegraphics[width=3.5in]{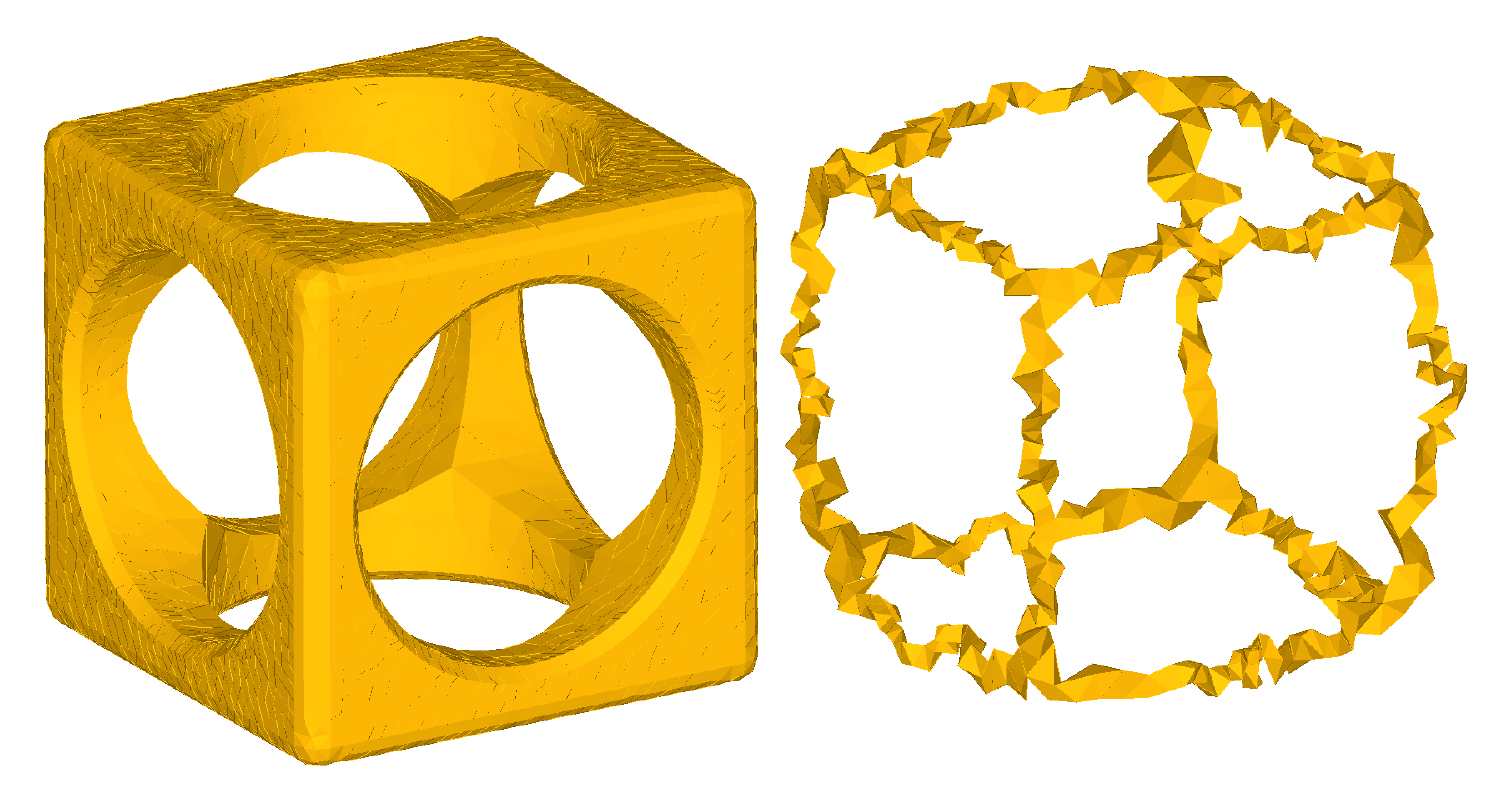}
\caption{A geometric figure and its topology-preserving skeleton.}\label{ex9}
\end{figure}
\begin{figure}[!t]
\centering
\includegraphics[width=3.5in]{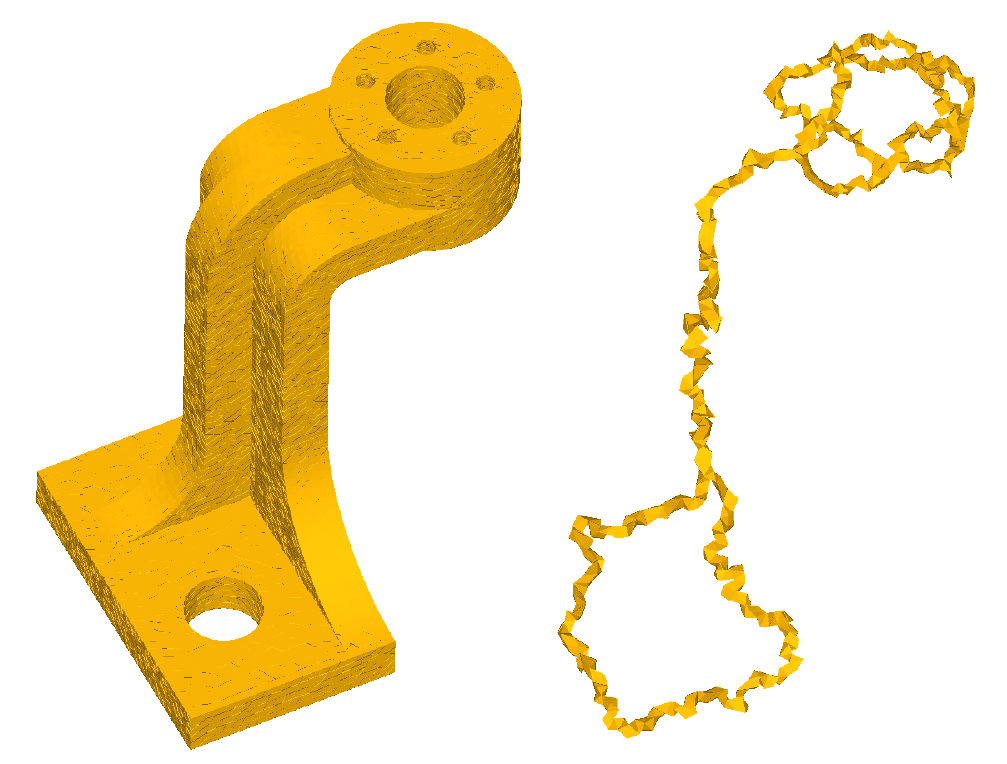}
\caption{A mechanical part and its topology-preserving skeleton.}\label{ex10}
\end{figure}
\begin{figure}[!t]
\centering
\includegraphics[width=3.5in]{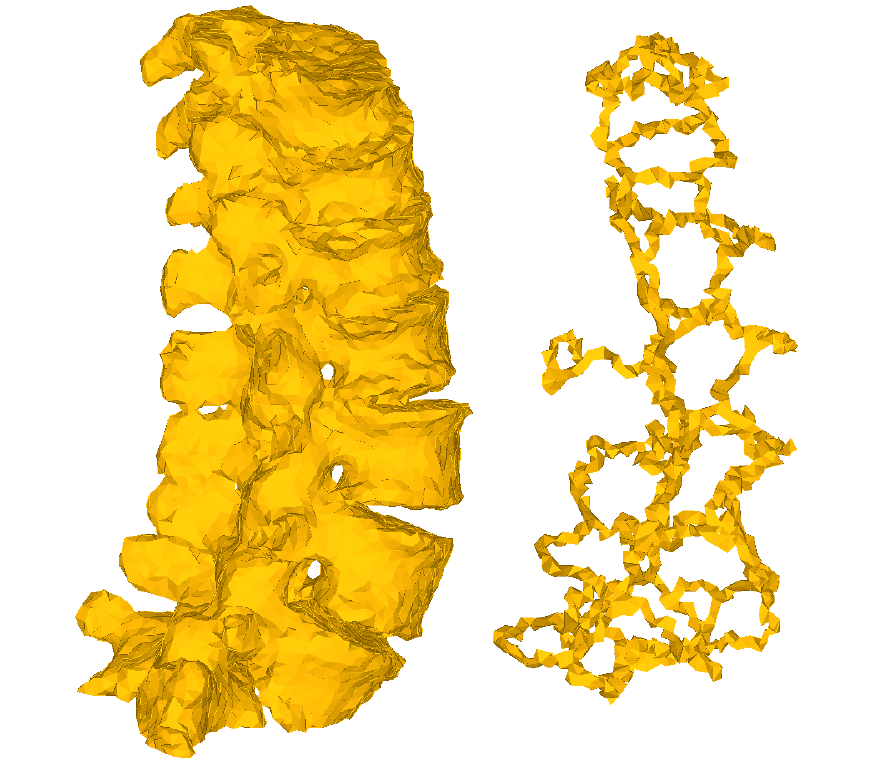}
\caption{A spine and its topology-preserving skeleton. This model is provided courtesy of IRCAD Institute (http://www.ircad.fr/) by the
AIM@SHAPE Shape Repository.}\label{ex14}
\end{figure}
\begin{figure}[!t]
\centering
\includegraphics[width=3.5in]{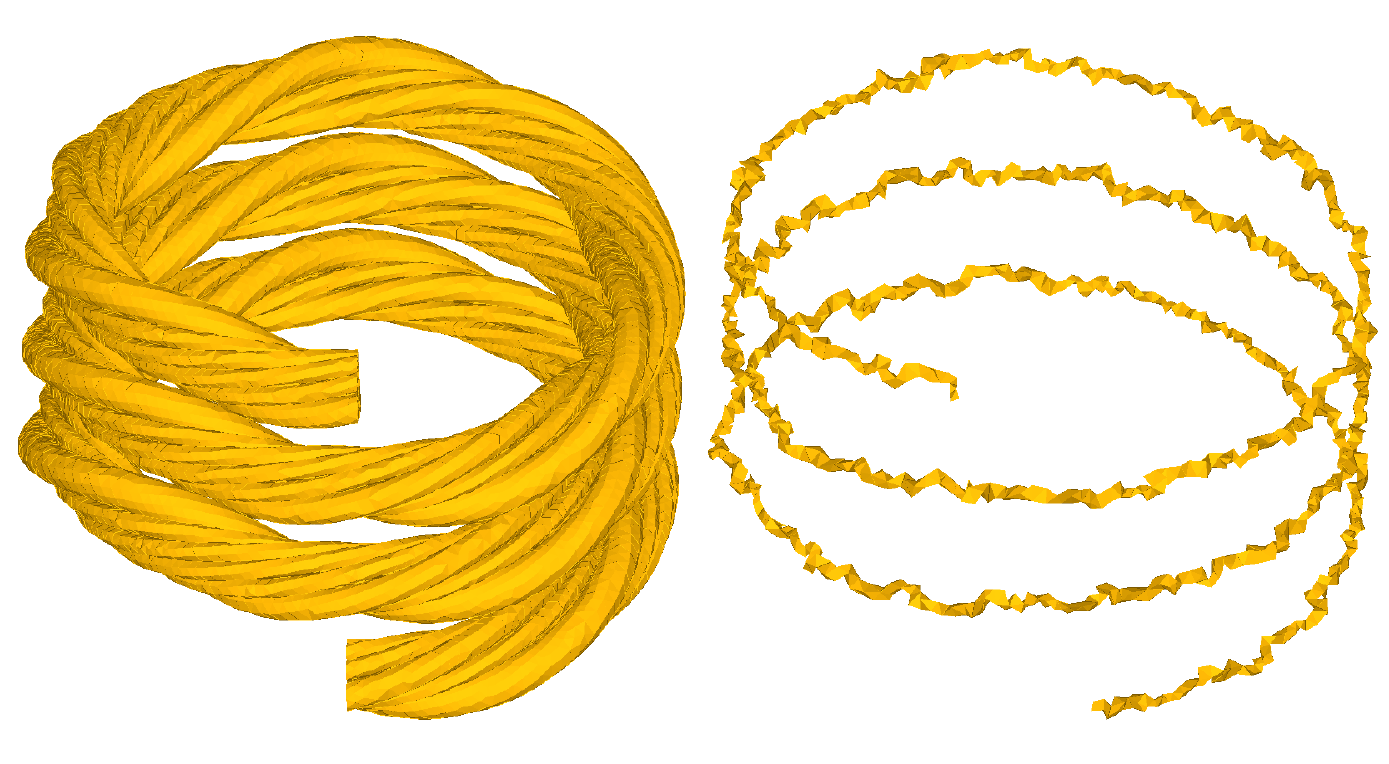}
\caption{A rope and its shape-preserving skeleton.}\label{ex5}
\end{figure}
\begin{figure}[!t]
\centering
\includegraphics[width=3.5in]{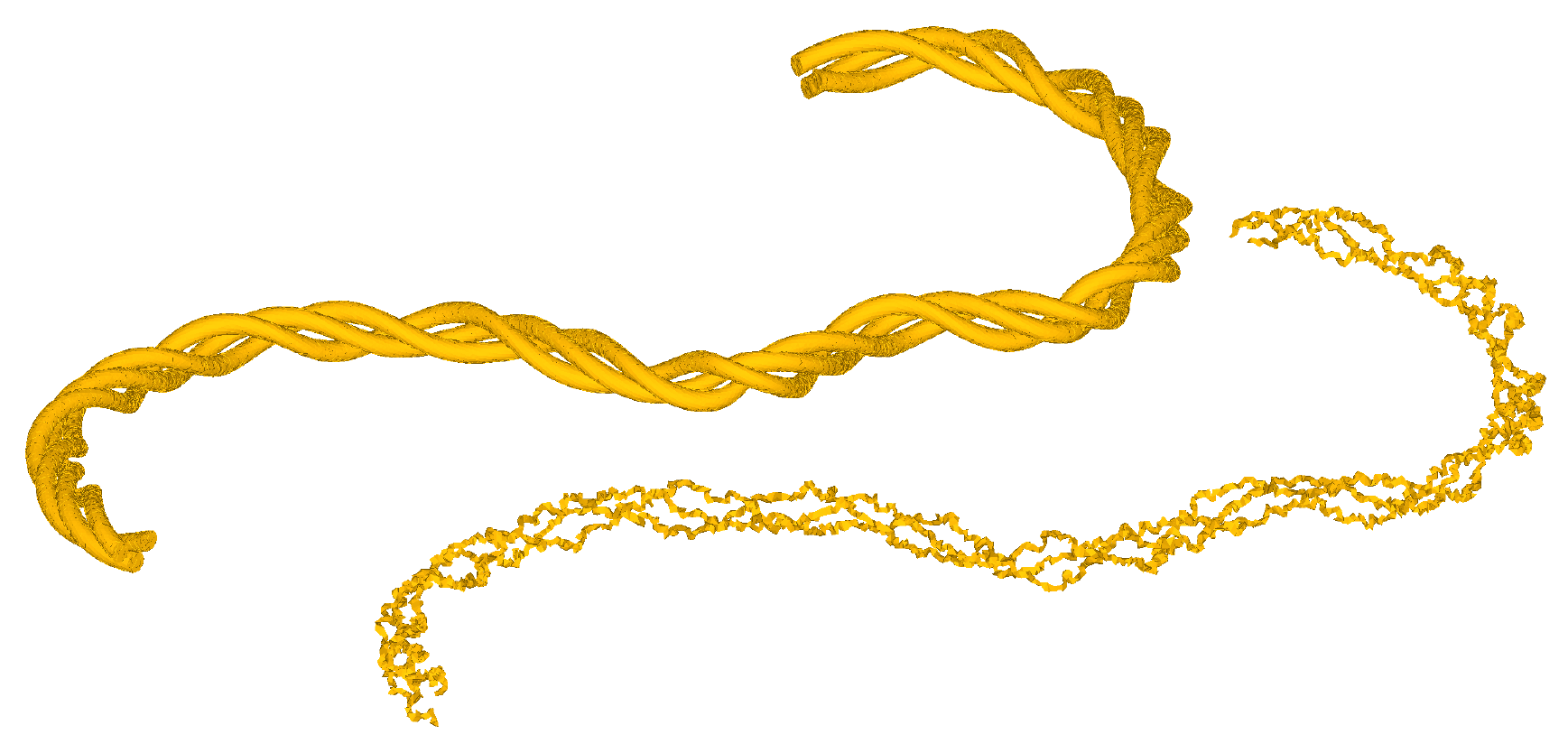}
\caption{Three twisted wires and their shape-preserving skeleton.}\label{ex6}
\end{figure}
\begin{figure}[!t]
\centering
\includegraphics[width=3.5in]{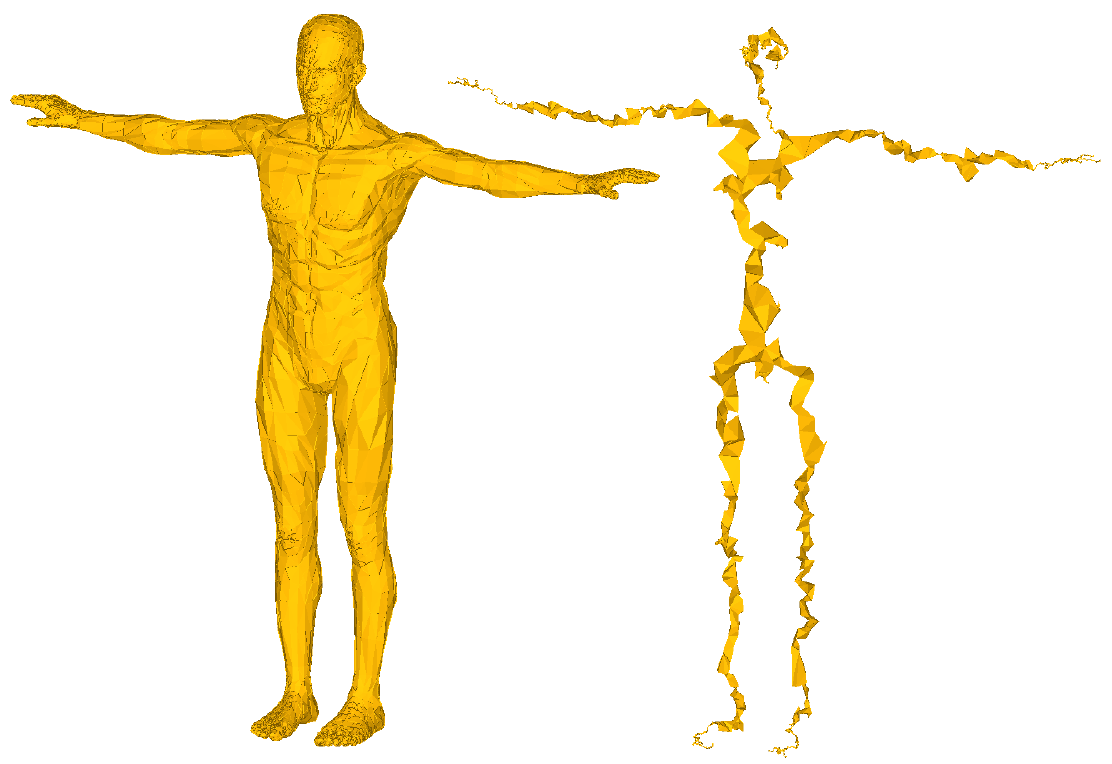}
\caption{A male body and its shape-preserving skeleton.}\label{ex12}
\end{figure}
\begin{figure}[!t]
\centering
\includegraphics[width=3.5in]{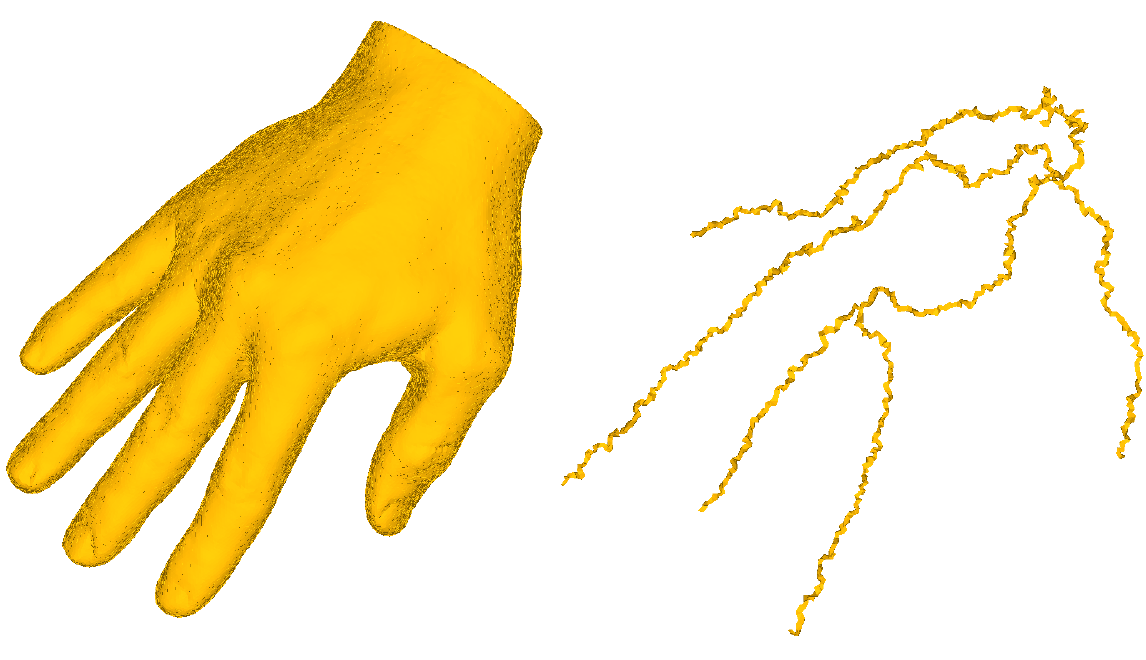}
\caption{A hand and its shape-preserving skeleton. This model is provided courtesy of INRIA (http://www.inria.fr/) and IMATI (http://www.imati.cnr.it/) by the
AIM@SHAPE Shape Repository.}\label{ex13}
\end{figure}

\section{Conclusion}
\label{sec:conclusions}
This paper introduces a topology preserving thinning algorithm for cell complexes based on iteratively culling simple cells. Simple cells, that may be seen as a generalization of simple points in digital topology, are characterized with homology theory. Despite homotopy, homology theory has the virtue of being computable. It means that, instead of resorting to complicated rule-based approaches, one can detect simple cells with homology computations.
The main idea of this paper is to give a classification of all possible simple cells that can occur in a cell complex with acyclicity tables.
These tables are filled in advance automatically by means of homology computations for all possible configurations. Once the acyclicity tables are available, implementing a thinning algorithm does not require any prior knowledge of homology theory or being able to compute homology.
The fact that acyclicity tables are filled automatically and correctly for all possible configurations provides a rigorous computer-assisted mathematical proof that the homology-based thinning algorithm preserves topology. We believe that such rigorous topological tools simplify the study of thinning algorithms and provide a clear and safe way of obtaining skeletons.

\end{document}